%% file: neurips_2026.tex
\documentclass{article}

\usepackage[preprint]{neurips_2026}

\usepackage[utf8]{inputenc} 
\usepackage[T1]{fontenc}    
\usepackage[table]{xcolor}  
\usepackage{hyperref}       
\usepackage{url}            
\usepackage{booktabs}       
\usepackage{amsmath}        
\usepackage{amsthm}         
\usepackage{amsfonts}       
\usepackage{nicefrac}       
\usepackage{microtype}      
\usepackage{amssymb}
\usepackage{makecell}
\usepackage{colortbl}

\usepackage{caption}

\usepackage{graphicx}
\usepackage{algorithm}
\usepackage{algorithmic}
\theoremstyle{plain}
\newtheorem{theorem}{Theorem}[section]

\theoremstyle{definition}
\newtheorem{definition}[theorem]{Definition}

\theoremstyle{remark}

\hypersetup{
	colorlinks=true,    
	linkcolor=red,      
	citecolor=blue,    
	filecolor=magenta,  
	urlcolor=magenta    
}

\title{Proximal Action Replacement for Behavior Cloning Actor-Critic in Offline Reinforcement Learning}

\author{%
  Jinzong Dong$^{1,2}$ \quad Wei Huang$^{2}$ \quad Jianshu Zhang$^{2,3}$ \quad Zhuo Chen$^{2,3}$ \\
  \bf Xinzhe Yuan$^{2}$ \quad Qinying Gu$^{2,\dagger}$ \quad Zhaohui Jiang$^{1,2}$ \quad Nanyang Ye$^{2,3,\dagger}$ \\
  \\
  $^{1}$School of Automation, Central South University, Changsha, China \\
  $^{2}$Shanghai AI Laboratory, Shanghai, China \\
  $^{3}$Shanghai Jiao Tong University, Shanghai, China \\
  \\
$^{\dagger}$\begin{tabular}[t]{@{}l@{}}
  Correspondence to: Qinying Gu \texttt{<guqinying@pjlab.org.cn>}, \\
  \phantom{Correspondence to: }Zhaohui Jiang \texttt{<jzh0903@csu.edu.cn>}, \\
  \phantom{Correspondence to: }Nanyang Ye \texttt{<ynylincolncam@gmail.com>}
\end{tabular}
}

\begin{document}

\maketitle
\begin{abstract}
Offline reinforcement learning (RL), which optimizes policies using a previously collected static dataset, is an important branch of RL. A popular and promising approach is to regularize actor-critic methods with behavior cloning (BC), which quickly yields realistic policies and mitigates bias from out-of-distribution actions, but it can impose an often-overlooked performance ceiling: when dataset actions are suboptimal, indiscriminate imitation structurally prevents the actor from fully exploiting better actions suggested by the value function, especially in later training when imitation is already dominant. We formally analyzed this limitation by investigating convergence properties of BC-regularized actor-critic optimization and verified it on a controlled continuous bandit task. To break this ceiling, we propose \textit{proximal action replacement} (PAR), an easy-to-use plug-and-play training sample replacer. PAR substitutes suboptimal dataset actions with better actions generated by a stable target policy, guided by the action-value function's local ascent direction and bounded by value uncertainty to ensure training stability. PAR is compatible with multiple BC regularization paradigms. Extensive experiments across offline RL benchmarks show that PAR consistently improves performance, and approaches state-of-the-art results simply by being combined with the basic TD3+BC.
\end{abstract}

\section{Introduction}
\label{Introduction}
Reinforcement learning (RL) empowers agents to develop optimal decision-making strategies through real-time interactions and has achieved remarkable progress across diverse domains \citep{doi:10.1126/scirobotics.adi8022}. However, these interactions with real-world environments limit their practical application to critical scenarios, such as healthcare decision-making \citep{pmlr-v174-fatemi22a} and autonomous driving \citep{10015868}. To address this challenge, offline reinforcement learning, which learns from past experiences without online interactions, has emerged as a vital and rapidly growing area of research \citep{NEURIPS2023_26cce1e5}.

However, naively applying standard online RL methods to previously collected data leads to the well-known distribution shift problem \citep{BCQ}, causing the learned policy to fail during online deployment. A widely adopted family of solutions is behavior cloning (BC) regularization actor-critic methods, where the actor is explicitly constrained to stay close to dataset actions. Typically, it employs mean squared error (MSE) \citep{TD3_BC}, Kullback-Leibler (KL) divergence \citep{DBLP:journals/corr/abs-1911-11361}, or maximum likelihood estimation (MLE) \citep{IQL} to reduce the discrepancy between learned and dataset actions. Owing to its simplicity and strong performance, this behavior cloning regularization has been widely adopted, spanning offline RL approaches based on state-conditional generative modeling \citep{wang2023diffusion, EDP} as well as some emerging offline reinforcement learning methods \citep{SSAR, kim2025penalizing, nguyen2026onestep}.

Despite its effectiveness for stability, BC regularization can impose an often-overlooked performance ceiling when the dataset contains suboptimal actions. Intuitively, once the policy has learned to generate “realistic” actions, continued pressure to imitate suboptimal samples keeps pulling the actor toward suboptimal regions, even when the critic has already identified better actions. This creates a structural trade-off between imitation and optimality, which becomes increasingly limiting in later training when the main bottleneck is no longer realism but exploring high-value actions.

\begin{figure}[t]
	\begin{center}
		\begin{minipage}{0.975\textwidth}
		\includegraphics[width=\columnwidth]{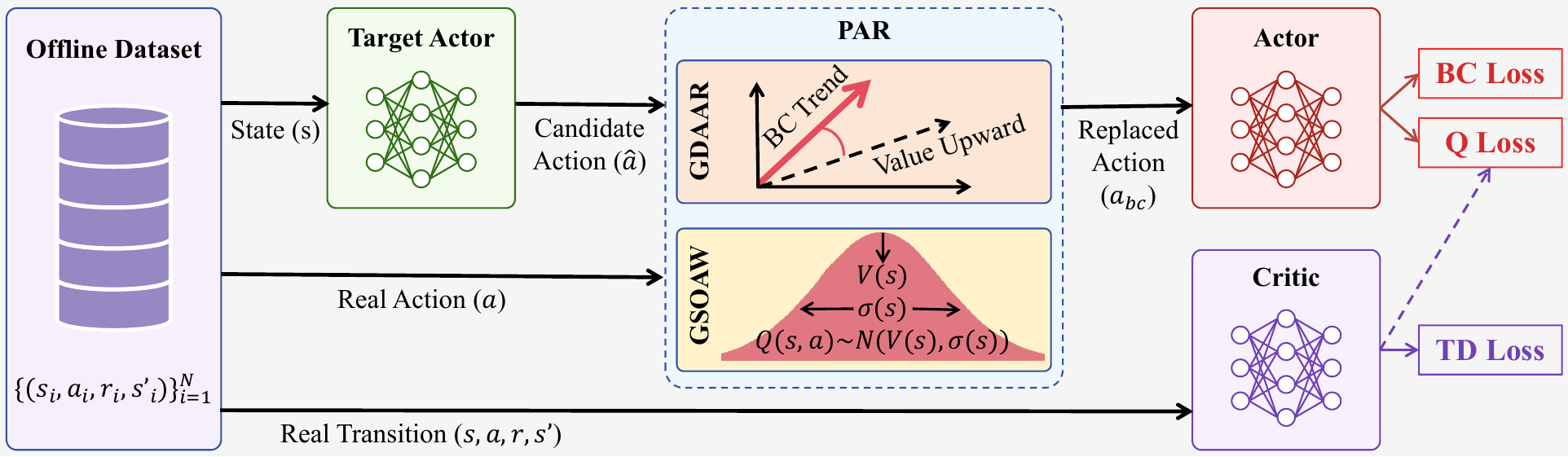}
		\end{minipage}
		\hfill
		\begin{minipage}{0.32\textwidth}
		\includegraphics[width=\columnwidth]{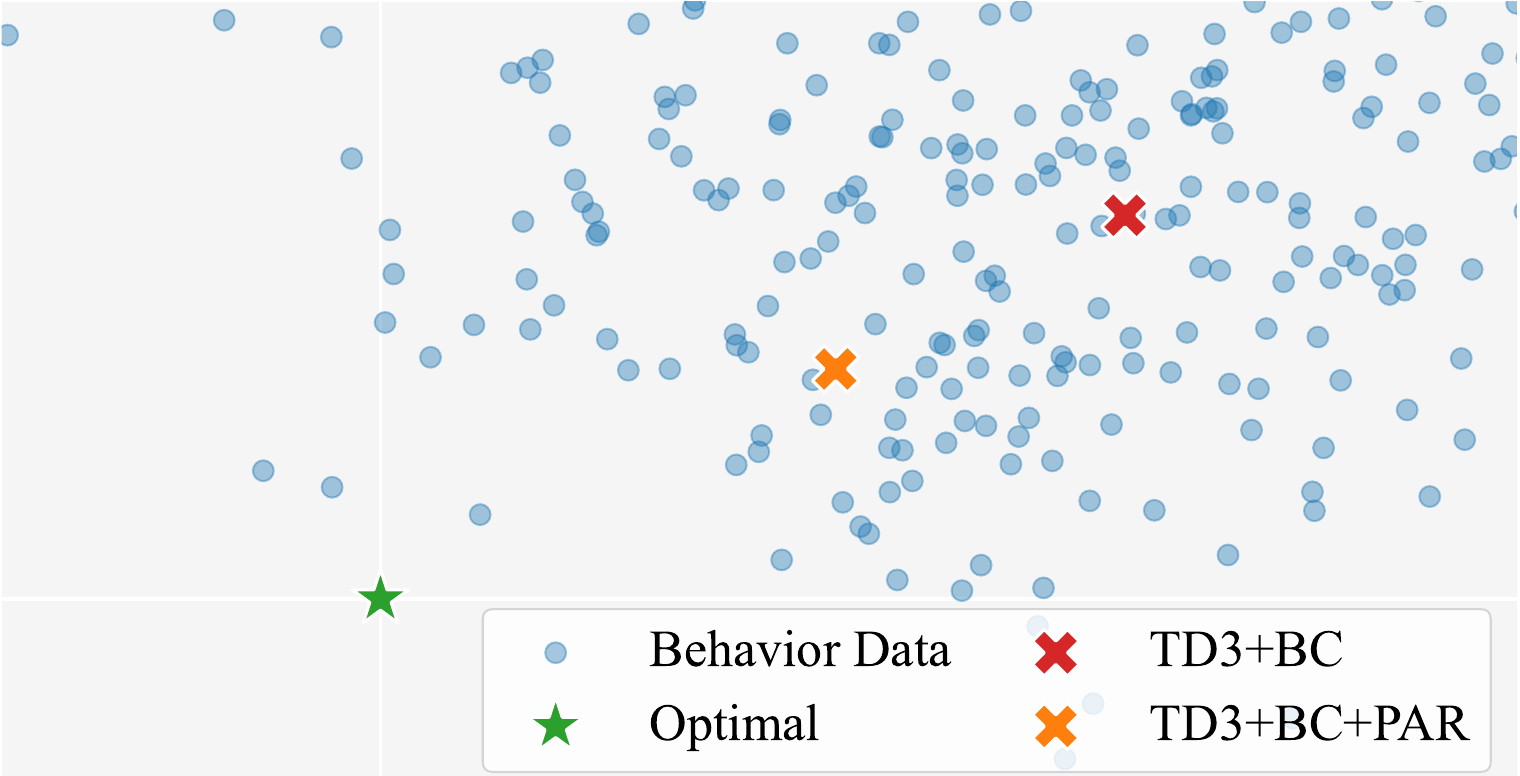}
		\end{minipage}
		\begin{minipage}{0.32\textwidth}
		\includegraphics[width=\columnwidth]{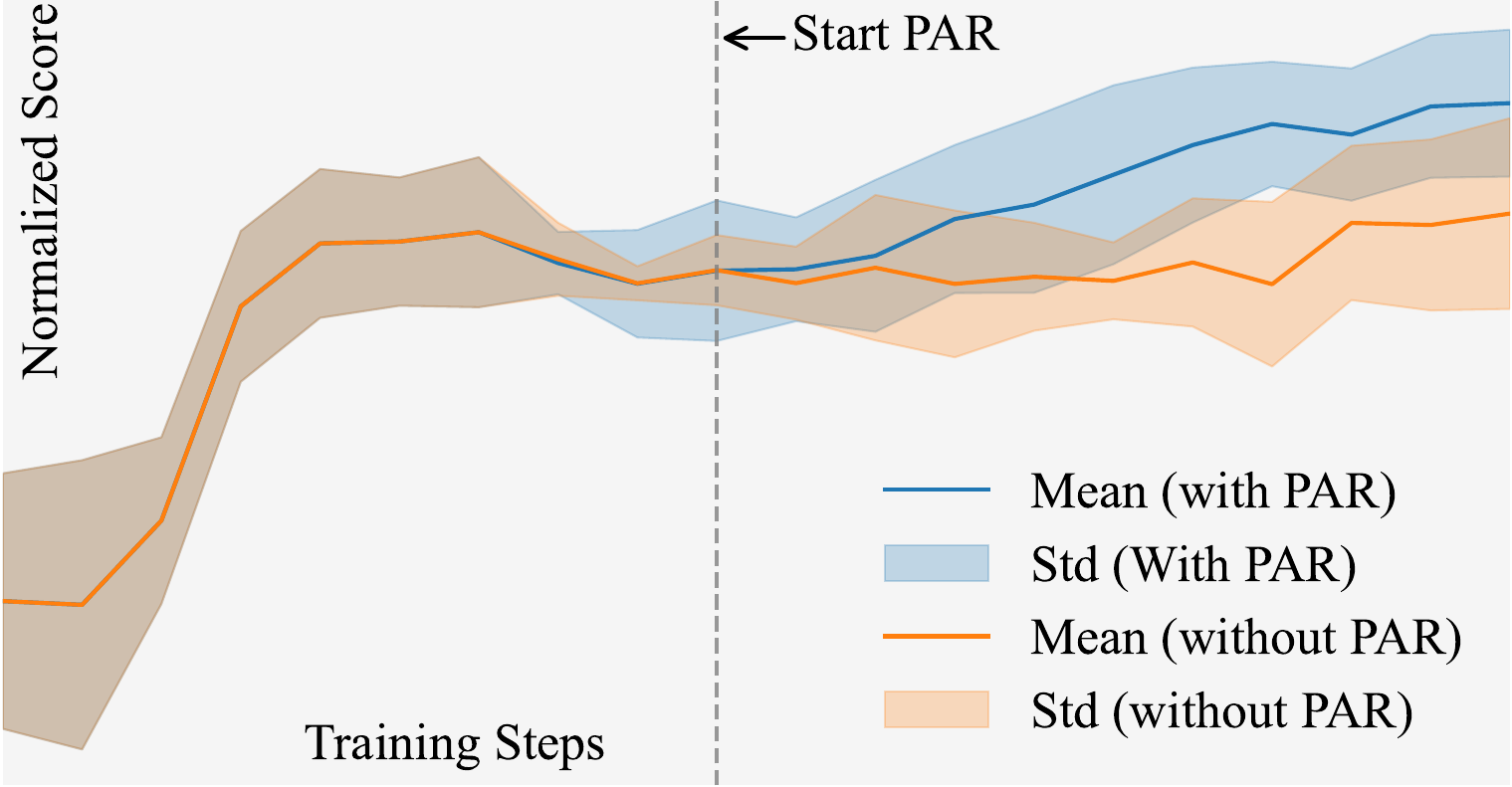}
		\end{minipage}
		\begin{minipage}{0.32\textwidth}
		\includegraphics[width=\columnwidth]{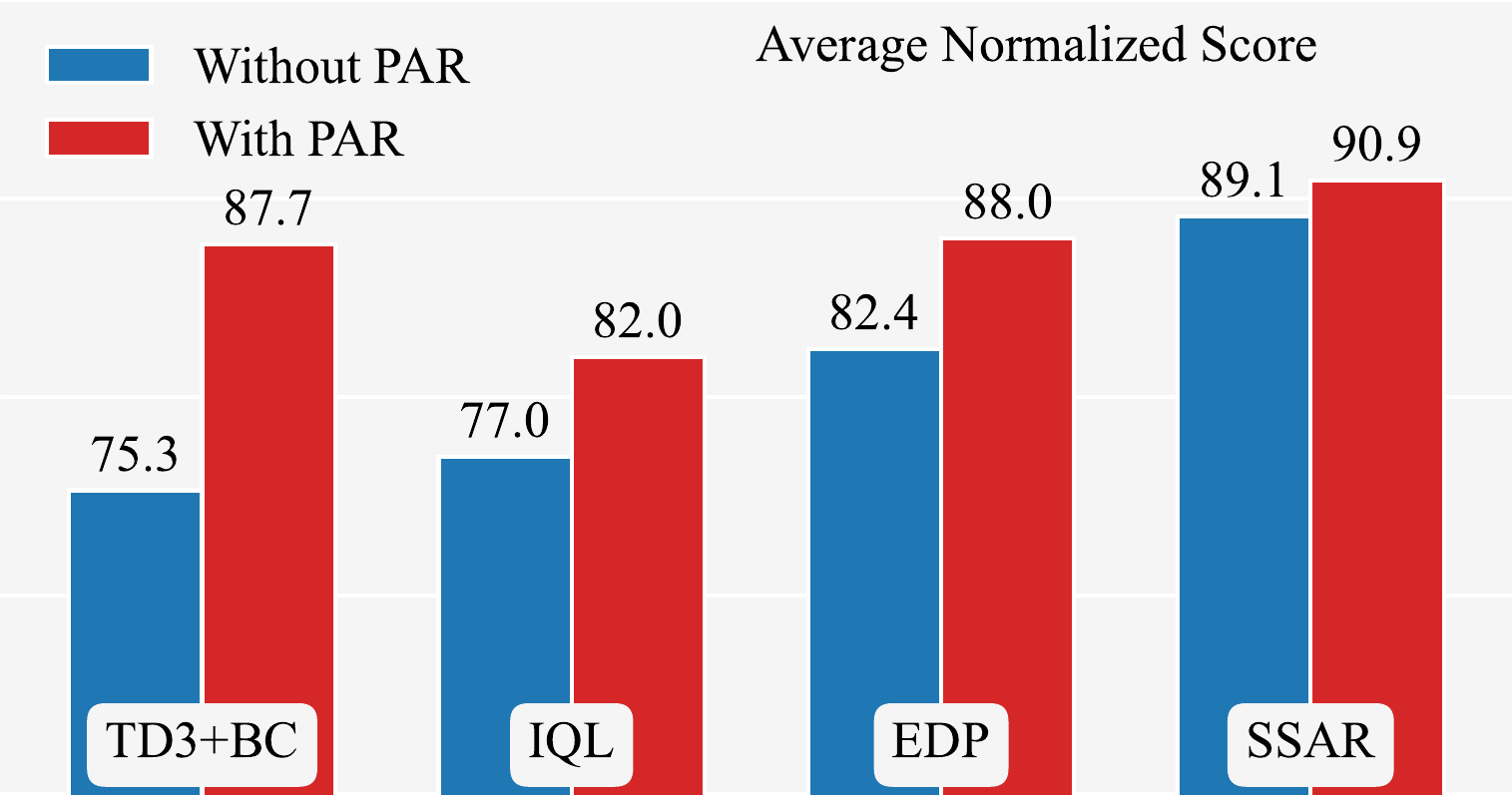}
		\end{minipage}
		\caption{PAR flowchart and effect display. \textit{Top}: PAR's plug-and-play flowchart. \textit{Bottom Left}: The policy obtained by TD3+BC+PAR is closer to the optimal policy than that obtained by TD3+BC on a continuous bandit task. \textit{Bottom Center}: Normalized score comparison during the training process between TD3+BC and TD3+BC+PAR on Walker2d-Medium-Replay. \textit{Bottom Right:} Average normalized score improvement of PAR across classic and advanced behavior cloning actor-critic methods in MuJoCo.}
		\label{fig:side:intro}
	\end{center}
\end{figure}

To formally characterize this limitation, we theoretically and empirically characterize that BC regularization inherently leads to suboptimal policies by analyzing the convergence properties of the actor-critic optimization. To address this fundamental issue, we propose proximal action replacement (PAR), a plug-and-play training sample replacer that dynamically replaces suboptimal BC supervision targets with high-value actions from a stable target policy, filtered by alignment with the state-action value function's local ascent gradient to ensure only genuinely better actions are selected. Furthermore, proximal constraints preserve training stability, enabling policies to transcend the performance ceiling imposed by BC regularization without sacrificing its benefits. PAR is compatible with various BC regularization forms (MSE, KL, MLE) and seamlessly integrates into existing offline RL algorithms. Extensive experiments demonstrate that PAR consistently improves performance across multiple algorithms and diverse domains, as illustrated in Fig. \ref{fig:side:intro}.

Our contributions are as follows: 1) We formally characterize the structural suboptimality of BC regularization through convergence analysis, proving that when data contains suboptimal actions, the stationary point of the BC-regularized objective cannot coincide with the true optimal policy, and empirically verify this on a controlled bandit task; 2) We propose Proximal Action Replacement (PAR) to break this performance ceiling, which consists of two key components: gradient-direction-aware action replacement that dynamically replaces suboptimal actions guided by the state-action value function's local ascent direction, and value-uncertainty-aware OOD action weighting that adaptively bounds action deviations based on value uncertainty to ensure training stability; 3) Extensive experiments demonstrate that PAR, as a plug-and-play module, consistently improves the performance of various BC-regularized offline RL algorithms across diverse domains.

\section{Background and Related Work}
\textbf{Offline RL}: The RL problem is typically defined by a Markov decision process (MDP): $M=\{S, A, P, R, \gamma, d_{0}\}$, with state space $S$, action space $A$, environment dynamics $P: S \times A \to \Delta(S)$, reward function $R: S \times A \to \mathbb{R}$, discount factor $\gamma$, and initial state distribution $d_{0}$, where $\Delta(S)$ represents the set of all probability distributions on the state space $S$. The goal of RL is to learn policy $\pi_{\theta}(a|s)$, parameterized by $\theta$, that maximizes the cumulative discounted reward $\mathbb{E}_{s_{0}\sim d_{0}, a_{t}\sim \pi(\cdot|s_{t}), s_{t+1}\sim P(\cdot|s_{t},a_{t})}\left[\sum_{t=0}^{\infty} \gamma^{t}R(s_{t}, a_{t})\right]$. The action-value (Q-value) of a policy $\pi$ is defined as $Q^{\pi}(s,a)=\mathbb{E}_{\pi}\left[\sum_{t=0}^{\infty}\gamma^{t}R(s_{t},a_{t})\mid s_{0}=s, a_{0}=a\right]$. In the offline RL setting, instead of the environment, a static dataset $\mathcal{D} \triangleq \{(s_{t},a_{t},r_{t},s_{t+1})\}$ is provided. Offline RL algorithms learn a policy solely from the fixed offline dataset $\mathcal{D}$, without any online interaction with the environment.

\textbf{Behavior Cloning Regularization}: Behavior cloning regularization helps models rapidly acquire realistic policies while mitigating bias from out-of-distribution actions, by constraining them to closely align with the behavior policies present in the dataset. Common behavior cloning regularizations include three forms: mean squared error (MSE) \citep{TD3_BC}, Kullback-Leibler (KL) divergence \citep{DBLP:journals/corr/abs-1911-11361}, or maximum likelihood estimation (MLE) \citep{IQL}, which are defined as follows.

\begin{definition}
\textnormal{\textbf{(MSE Behavior Cloning)}: }For a deterministic policy $\pi_{\theta}: S \rightarrow A$, the MSE regularization minimizes the squared Euclidean distance between the policy's action and the action in the dataset $\mathcal{D}$: $\mathcal{L}_{\text{MSE}}(\theta) = \mathbb{E}_{(s,a)\sim\mathcal{D}}[\|\pi_{\theta}(s) - a\|_2^2]$.
\end{definition}

\begin{definition}
	\textnormal{\textbf{(KL Behavior Cloning)}: }For a stochastic policy $\pi_{\theta}(\cdot|s)$, the KL regularization minimizes the Kullback-Leibler divergence between $\pi_{\theta}$ and the behavior policy $\pi_{\beta}$ estimated from the dataset: $\mathcal{L}_{\text{KL}}(\theta) = \mathbb{E}_{s\sim\mathcal{D}}[D_{\text{KL}}(\pi_{\theta}(\cdot|s) \| \pi_{\beta}(\cdot|s))]$.
\end{definition}

\begin{definition}
	\textnormal{\textbf{(MLE Behavior Cloning)}: }For a stochastic policy $\pi_{\theta}(\cdot|s)$, the MLE regularization maximizes the weighted log-likelihood of the actions in the dataset: $\mathcal{L}_{\text{MLE}}(\theta) = \mathbb{E}_{(s,a)\sim\mathcal{D}}[-w(s,a) \log \pi_{\theta}(a|s)]$, where $w(s,a)$ is a weighting term often derived from Q-values or advantages.
\end{definition}

\textbf{Related Work:} Appendix \ref{RelatedWork} provides a detailed analysis of related works on offline RL. In contrast to these methods, PAR takes a fundamentally different perspective: instead of modifying policy constraints, value functions, or policy representations, PAR addresses the suboptimality issue at the data level by replacing suboptimal actions in the training batch with better actions generated by a stable actor. Unlike simple action weighting methods, PAR actively generates and injects superior actions into training—making performance gains both easier to achieve and larger in magnitude, and thereby more effectively overcoming the ceiling imposed by BC regularization.

\section{Method}
\label{Method}
In this section, we first show that standard BC regularization is structurally suboptimal under offline data containing suboptimal transitions, and then introduce Proximal Action Replacement (PAR), which selectively replaces suboptimal BC supervision targets with better target-policy actions according to critic-guided direction alignment and proximal OOD-aware constraints, enabling safer policy improvement beyond imitation-limited solutions.

\subsection{BC Regularization Leads to Suboptimal Policy}
To understand the fundamental impact of BC regularization, we analyze the optimization dynamics from the perspective of distinguishing between the \textit{computed optimal solution} (obtained by optimizing the regularized objective that balances Q-value maximization with behavior cloning constraints) and the \textit{ideal optimal solution} (the true optimal policy that maximizes expected return). The formal description is presented in Theorem \ref{prop:suboptimality}.
\begin{theorem} \label{prop:suboptimality}
\textnormal{\textbf{(Sub-optimality of BC Regularization)}} Let $(\hat{\pi}, \hat{Q})$ be the computed optimal pair for the regularized objective $J_{\theta}(\pi_{\theta}, Q) = \mathbb{E}[\lambda Q(s, \pi_\theta) - \mathcal{L}_{\text{MSE}}(\pi_{\theta}, a_{\text{data}})]$ and the Bellman error. Let $(\pi^*, Q^*)$ denote the ideal optimal pair where $\pi^*(s) = \arg\max_a Q^*(s, a)$.
	If the dataset action $a_{\text{data}}$ is sub-optimal (i.e., $a_{\text{data}} \neq \pi^*(s)$) and $\lambda < \infty$, then the converged policy $\hat{\pi}$ is strictly sub-optimal:
	\begin{equation}
		(\hat{\pi}, \hat{Q}) \neq (\pi^*, Q^*) \quad \text{and} \quad Q^*(s, \hat{\pi}(s)) < Q^*(s, \pi^*(s)).
	\end{equation}
	Its proof is provided in Appendix \ref{proof_of_3_1}.
\end{theorem}

\textbf{Remark on Theorem \ref{prop:suboptimality}:} The proof of Theorem \ref{prop:suboptimality} can be easily derived from the property of stationary points, namely that $(\hat{\pi}, \hat{Q})$ must be a stationary point of $J_{\theta}(\pi_{\theta}, Q)$. However, since $\nabla_\theta \mathcal{L}_{\text{BC}}(\pi^*, a_{\text{data}}) \ne 0$, $(\pi^*, Q^*)$ must not be a stationary point of $J_{\theta}(\pi_{\theta}, Q)$, $(\hat{\pi}, \hat{Q}) \ne (\pi^*, Q^*)$. Theorem \ref{prop:suboptimality} tells us that indiscriminate imitation of suboptimal data fundamentally constrains a policy’s asymptotic performance. This limitation may not be evident during the early stages of training, but it becomes a critical factor hindering performance improvements once the policy network learns to generate sufficiently realistic actions. Similar conclusions also apply to behavior cloning regularization in both the KL divergence and maximum likelihood forms, as detailed in Appendix \ref{extension_of_3_1}.

\textbf{Toy Experiment:} 
To empirically validate Theorem \ref{prop:suboptimality}, we design a 2D continuous bandit task with oracle value $Q(s,\mathbf{a})=-\|\mathbf{a}\|_2^2$ and optimal action $\mathbf{a}^*=[0,0]^\top$, while constructing an offline dataset from a biased suboptimal behavior distribution centered at $[2,2]^\top$. Using TD3+BC with MSE/KL/MLE-style BC regularization, we observe that although MLE tends to move the learned policy closer to $\mathbf{a}^*$, all standard BC variants still converge to compromise solutions rather than the true optimum, empirically confirming the structural suboptimality predicted by Theorem \ref{prop:suboptimality}. Detailed settings and visualizations are provided in Appendix \ref{proof_of_3_1_toy_experiment}.

\subsection{Proximal Action Replacement}
To break the imitation-induced performance ceiling while preserving offline training stability, we propose \textit{Proximal Action Replacement} (PAR). PAR performs on-the-fly relabeling only for the actor's BC supervision target: for each state, it selectively replaces the logged action with a target-policy action when the latter is more consistent with the critic's local ascent direction, and controls distributional risk via proximal OOD-aware weighting. The details are given below.

\begin{figure}[t]
	\begin{center}
		\includegraphics[width=\columnwidth]{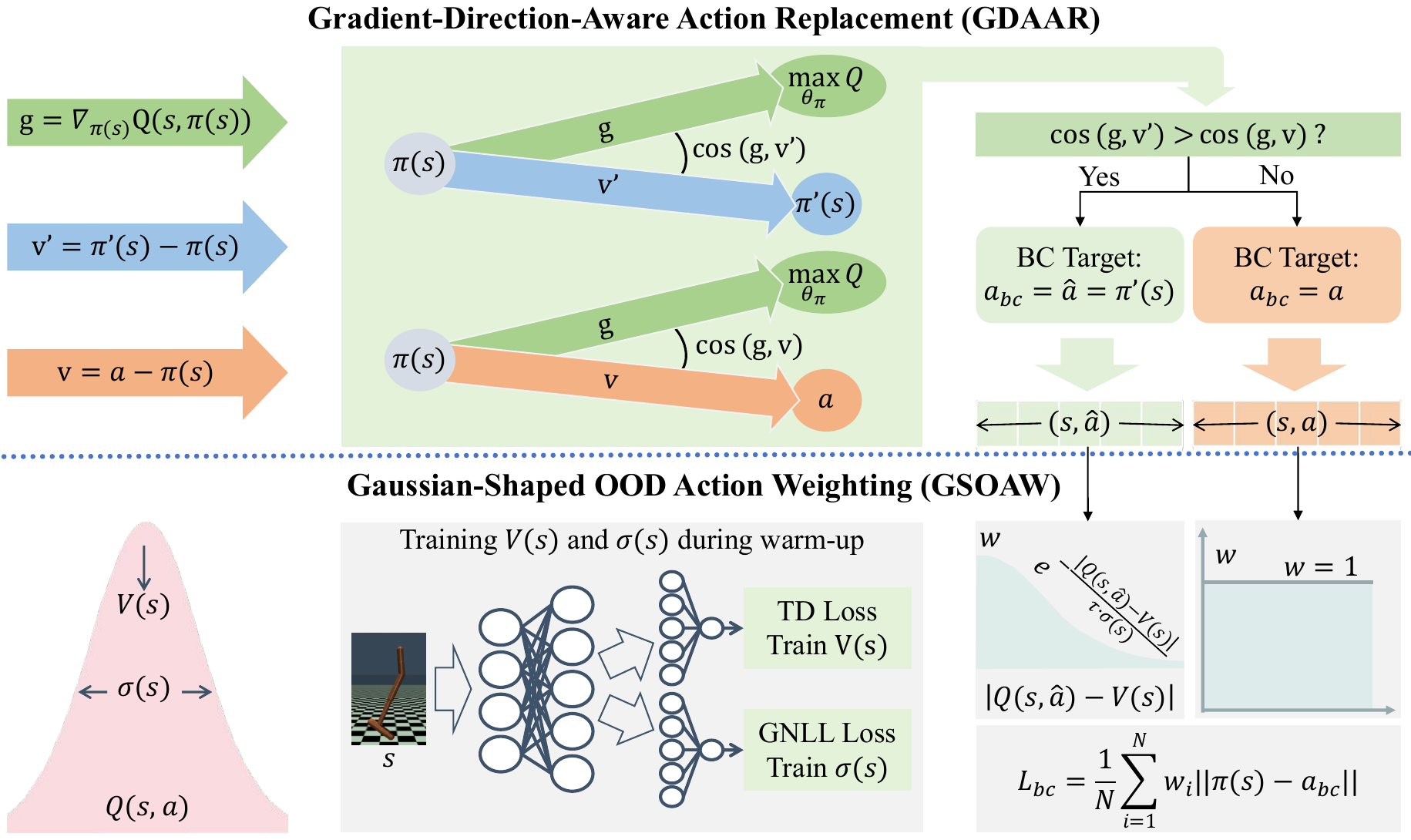}
		\caption{Overview of PAR. $\pi'(s)$ represents the target actor (a Polyak-averaged copy of actor), $s$ and $a$ represents offline real state and action, and $\sigma(s)$ represents the standard deviation of $Q(s,a)$ when $s$ is fixed. The pseudocode for the method is provided in Appendix \ref{pseudocode}.}
		\label{fig:side:a}
	\end{center}
\end{figure}

\subsubsection{Gradient-Direction-Aware Action Replacement}
The key question in PAR is how to decide, for each sample, whether the behavior cloning objective should still imitate the logged action $a$ or instead follow an action proposed by the learned policy. A misjudgment either perpetuates suboptimal imitation or destabilizes training. Intuitively, a superior action should yield a higher expected return, which is parameterized by the critic network. Thus, the critic's local ascent direction serves as a natural compass for identifying policy improvements. Therefore, we introduce gradient-direction-aware action replacement (GDAAR): we separately measure the alignment of the logged action and the target-policy action against the critic's local ascent direction, and only switch the BC supervision when the target-policy action is more ascent direction-consistent.

Specifically,  let $\pi_{\theta}(s)$ denote the current actor’s output and let $\pi'(s)$ denote the target actor (a Polyak-averaged copy of $\pi_{\theta}$), with $\hat a = \pi'(s)$ (Fig.~\ref{fig:side:a}). We compare two displacement vectors from $\pi_{\theta}(s)$: $v'=\hat{a}-\pi_{\theta}(s)$ and $v=a-\pi_{\theta}(s)$. Let $g=\nabla_{a}Q(s,a)\big|_{a=\pi_{\theta}(s)}$ be the action-space gradient of the critic at the current actor output. Intuitively, $g$ points along the direction of steepest ascent of $Q(s,\cdot)$ near $\pi_{\theta}(s)$. We measure alignment with this local value-improvement direction using cosine similarity,
\begin{equation}
	\cos(g,u) = \frac{g^{\top}u}{\|g\|_{2}\|u\|_{2}+\varepsilon},
\end{equation}
where $u$ is $v$ or $v'$, with a small $\varepsilon>0$ for numerical stability when $g$ or $u$ has small norm.

The BC regression label is
\begin{equation}
	\label{eq:gdar_target}
	a_{\mathrm{bc}}=
	\begin{cases}
		\hat{a}, & \text{if } \cos(g,v')>\cos(g,v),\\
		a, & \text{otherwise}.
	\end{cases}
\end{equation}
Thus, the actor’s BC loss is evaluated with $a_{\mathrm{bc}}$: we set the supervision to $\hat{a}=\pi'(s)$ only when the displacement induced by the target actor is more aligned with $g$ than the displacement induced by the logged action; if $a$ is already better aligned, we keep $a$ as the BC label, preserving a conservative anchor against unwarranted policy drift.

\begin{theorem}
\textnormal{\textbf{(Advantage Guarantee of GDAAR)}} 
\label{theo:GDAAR}
For a fixed state $s$, let $a_\theta=\pi_\theta(s)$, 
$v=a-a_\theta$, $v'=\hat a-a_\theta$, and 
$g=\nabla_a Q(s,a)\vert_{a=a_\theta}$.
Assume $Q(s,\cdot)$ is $L$-smooth, the actor uses a small BC step toward the chosen label, and the two candidate displacements are locally norm-comparable ($\|v'\|_2\approx\|v\|_2$).  
If GDAAR selects $\hat a$ (i.e., $\cos(g,v')>\cos(g,v)$), then the one-step value improvement is no worse than vanilla BC up to second-order error:
\begin{equation}
Q(s,a_\theta^{+,\mathrm{GDAAR}})
\;\ge\;
Q(s,a_\theta^{+,\mathrm{BC}})
\;-\;\mathcal{O}(\eta^2),
\end{equation}
where $\eta$ is the actor step size, and $a_\theta^{+,\mathrm{GDAAR}}$ and $a_\theta^{+,\mathrm{BC}}$ denote one actor update step under GDAAR and vanilla MSE behavior cloning, respectively. Its proof is provided in Appendix \ref{proof_of_GDAAR}.
\end{theorem}

\textbf{Remark on Theorem \ref{theo:GDAAR}:} Theorem \ref{theo:GDAAR} is a mechanism-level guarantee for the replacement rule in Eq.~\eqref{eq:gdar_target}: when the target-action displacement is better aligned with the local critic ascent direction, selecting $\hat a$ compares favorably with vanilla BC up to second-order optimization error. Similar conclusions also apply to behavior cloning regularization in both the KL divergence and maximum likelihood forms, as detailed in Appendix \ref{extension_of_3_2}.

\subsubsection{Gaussian-Shaped OOD Action Weighting}
\textbf{Necessity of Proximal Optimization:}
In offline actor-critic methods, the critic target explicitly depends on the current policy through $y = r + \gamma Q(s', \pi_{\theta}(s'))$. As a result, if the learned policy deviates substantially from the behavior policy, the target actions $\pi_{\theta}(s')$ become increasingly out-of-distribution relative to the offline dataset. This distribution shift can raise the achievable critic loss and make critic learning brittle, which in turn destabilizes the subsequent actor updates. Therefore, action replacement cannot change the actor too much, that is, the updated actor cannot be too far away from the original actor. To gain deep insight into the nature of this mechanism, we analyze the relationship between the distance of the learned policy from the behavior policy and the critic loss. A formal description is provided in Theorem \ref{theo:pao}.

\begin{theorem} \label{theo:pao}
\textnormal{\textbf{(Instability via Policy Divergence)}} 
Let $\mathcal{L}(Q) = \mathbb{E}_{\mathcal{D}}[(Q(s,a) - y)^2]$ be the critic loss with target $y = r + \gamma Q(s', \pi_{\theta}(s'))$, where $s'$ represents next state. The minimum achievable loss is lower-bounded by the squared policy divergence:
\begin{equation}
    \min_Q \mathcal{L}(Q) \ge \mu \cdot \mathbb{E}_{(s,a)\sim\mathcal{D}} \mathbb{E}_{s'|s,a} \left[ \|\pi_{\theta}(s') - \pi_{\beta}(s')\|_2^2 \right],
\end{equation}
where $\mu$ is the minimum eigenvalue of the covariance matrix $\text{Cov}(\nabla_a Q)$. Its proof is provided in Appendix \ref{proof_of_proximal}.
\end{theorem}

\textbf{Remark on Theorem \ref{theo:pao}:} Theorem \ref{theo:pao} states that if the learned policy is far from the behavior policy, the critic's training loss will increase or even collapse. Thus, a proximal constraint is necessary to bound the policy divergence, ensuring stable target approximation for valid Q-function learning.

\textbf{Gaussian-Shaped OOD Action Weighting:} 
In Eq. \eqref{eq:gdar_target}, the action generated by the target actor $\pi'(s)$ may often fall outside the support of the real action. This can cause the updated policy to deviate significantly from the original behavioral policy, thus leading to the critic's collapse (Theorem~\ref{theo:pao}). To solve this problem, we add a value-aware weight on the BC loss that shrinks whenever $|Q(s,\hat{a})-V(s)|$ is large relative to a learned uncertainty scale $\sigma(s)$, so imitation is automatically weakened at out-of-distribution generated action.

Specifically, let $V(s)$ be an auxiliary state value function,
\begin{equation}
	V(s)=\mathbb{E}_{a\sim p_{\mathcal D}(a\mid s)}\big[Q(s,a)\big], 
\end{equation}
so $V(s)$ is the mean of $Q(s, a)$ $w.r.t.$ $a$ in principle. As a tractable surrogate, we model the distribution of $Q(s,a)$ $w.r.t.$ $a$ as a Gaussian distribution with mean $V(s)$ and standard deviation $\sigma(s)$. Importantly, the auxiliary state value network $V_{\psi}(s)$ can be solved using the classic temporal difference algorithm \citep{Sutton1988}, while $\sigma_{\psi}(s)$ can be learned simply by adding a prediction head to the auxiliary state value network and training with the following loss,
\begin{equation}
	\label{eq:gnll_sigma}
	\mathcal{L}_{\sigma}
	=\mathbb{E}_{(s,a)\sim\mathcal{D}}\!\left[\frac{\big(Q_{\phi}(s,a)-\texttt{sg}[V_{\psi}(s)]\big)^{2}}{2\sigma_{\psi}(s)^{2}}+\log \sigma_{\psi}(s)\right],
\end{equation}
with $\sigma_{\psi}(s)$ enforced to be positive (e.g., softplus parameterization), and $\texttt{sg}$ represents stop-gradient. Therefore, the generated actions can be assigned the following weights to perceive OOD actions,
\begin{equation}
	\label{eq:ood_weight}
	w = \exp\!\left(-\,\frac{\bigl|Q(s,\hat{a})-V(s)\bigr|}{\tau\,\sigma(s)+\varepsilon}\right),
\end{equation}
where $\tau>0$ represents temperature, $\varepsilon>0$ avoids division by zero. The value of $w$ ranges between (0,1]. The farther $Q(s, \hat a)$ is from $V(s)$, the closer $w$ is to 0. Therefore, generated actions that would disrupt the stability of $Q(s, a)$ will have a small weight in behavior cloning. The final behavior cloning loss on the generated action is: 
\begin{equation}
	\label{eq:weighted_bc}
	\mathcal{L}_{\mathrm{bc}, \hat a} \;=\; \frac{1}{N_{\hat a}}\sum_{i=1}^{N_{\hat a}} w_i\, \bigl\|\pi_{\theta}(s_i)-\hat a_{i}\bigr\|_2^2,
\end{equation}
where $w_i$ denotes Eq.~\eqref{eq:ood_weight} evaluated on sample $(s_{i}, \hat a_{i})$. On real offline data (s, a), the behavior cloning loss remains unchanged, i.e., $w\equiv 1$. 
\begin{theorem}
\textnormal{\textbf{(Stability Guarantee of GSOAW)}} 
\label{theo:GSOAW}
Let $t$ denote the training iteration index, and let $w_t(s)$ be the GSOAW weight at iteration $t$ defined in Eq.~\eqref{eq:ood_weight}. Assume the actor is optimized with Eq.~\eqref{eq:weighted_bc} on generated actions and with the original (unweighted) BC loss on real offline actions $(s,a)\sim\mathcal D$, and that the critic loss is continuous with respect to action perturbation in a neighborhood of the behavior policy action. Then for any $\epsilon>0$, there exists $\tau>0$ such that
\begin{equation}
\left|\mathcal{L}_t(Q_{\phi_t};\pi_{\theta_{t+1}})-\mathcal{L}_t(Q_{\phi_t};\pi_{\beta})\right|\le \epsilon,
\end{equation}
where $\mathcal{L}_t(Q_{\phi_t};\pi)$ denotes the critic Bellman regression loss expectation at iteration $t$ under policy $\pi$. Its proof is provided in Appendix \ref{proof_of_GSOAW}.
\end{theorem}
\textbf{Remark on Theorem \ref{theo:GSOAW}:} Theorem \ref{theo:GSOAW} formalizes that GSOAW can make the critic update under the learned actor arbitrarily close to that under the behavior policy by choosing an appropriate temperature $\tau$. Intuitively, when generated actions are potentially out-of-distribution, the weight in Eq.~\eqref{eq:ood_weight} suppresses their contribution in Eq.~\eqref{eq:weighted_bc}, thereby limiting harmful action perturbations to the Bellman target. Consequently, GSOAW provides a principled stability mechanism that complements Theorem~\ref{theo:pao}: instead of enforcing a hard proximity constraint, it adaptively down-weights risky generated actions according to value deviation and uncertainty. Similar conclusions also apply to KL and MLE behavior cloning, as detailed in Appendix \ref{extension_of_3_4}.

\section{Results}
\label{Results}

\begin{table*}[t]
	\centering
	\caption{Comparison of normalized scores on D4RL benchmark across MuJoCo, Kitchen, and AntMaze domains. Bold values indicate that performance was improved after using PAR.}
	\setlength\tabcolsep{1.0pt}
	\renewcommand{\arraystretch}{1}
	\footnotesize
	\begin{tabular}{l|cc|cc|cc|cc|cc}
		\toprule
		\textbf{Dataset} & \textbf{BCQ} & \makecell{\textbf{BCQ}\\ \textbf{+PAR}} & \makecell{\textbf{TD3}\\\textbf{+BC}} & \makecell{\textbf{TD3+BC}\\ \textbf{+PAR}} & \textbf{IQL} & \makecell{\textbf{IQL}\\ \textbf{+PAR}} & \textbf{EDP} & \makecell{\textbf{EDP}\\ \textbf{+PAR}} & \textbf{SSAR} & \makecell{\textbf{SSAR}\\ \textbf{+PAR}}\\
		\midrule
		HalfCheetah-M-E &-&-&90.7&\textbf{95.1$_{\pm 0.8}$}&86.7&\textbf{93.4$_{\pm 1.0}$}&95.8$_{\pm 0.1}$&93.1$_{\pm 0.6}$&94.9$_{\pm 1.2}$&\textbf{96.2$_{\pm 0.4}$}\\
		Hopper-M-E &-&-&98.0&\textbf{111.6$_{\pm 1.3}$}&91.5&\textbf{108.5$_{\pm 6.5}$}&110.8$_{\pm 0.4}$&\textbf{111.1$_{\pm 1.0}$}&103.8$_{\pm 6.7}$&\textbf{112.5$_{\pm 0.3}$}\\
		Walker2d-M-E &-&-&110.1&\textbf{111.7$_{\pm 0.4}$}&109.6&\textbf{111.0$_{\pm 0.5}$}&110.4$_{\pm 0.0}$&\textbf{112.1$_{\pm 0.3}$}&112.5$_{\pm 1.4}$&111.5$_{\pm 0.5}$\\
		\midrule
		HalfCheetah-M &-&-&48.3&\textbf{49.5$_{\pm 0.2}$}&47.4&\textbf{48.3$_{\pm 0.1}$}&50.8$_{\pm 0.0}$&\textbf{55.1$_{\pm 0.4}$}&56.5$_{\pm 3.7}$&\textbf{58.3$_{\pm 1.9}$}\\
		Hopper-M &-&-&59.3&\textbf{101.8$_{\pm 2.6}$}&66.3&\textbf{71.6$_{\pm 4.1}$}&72.6$_{\pm 0.2}$&\textbf{86.1$_{\pm 7.4}$}&101.6$_{\pm 0.4}$&\textbf{102.7$_{\pm 0.5}$}\\
		Walker2d-M &-&-&83.7&\textbf{86.0$_{\pm 0.5}$}&78.3&\textbf{82.3$_{\pm 4.6}$}&86.5$_{\pm 0.2}$&\textbf{89.7$_{\pm 0.4}$}&87.9$_{\pm 2.4}$&\textbf{89.1$_{\pm 0.8}$}\\
		\midrule
		HalfCheetah-M-R &-&-&44.6&\textbf{45.4$_{\pm 0.5}$}&44.2&\textbf{44.7$_{\pm 0.2}$}&44.9$_{\pm 0.4}$&\textbf{48.1$_{\pm 1.4}$}&49.6$_{\pm 0.3}$&\textbf{49.9$_{\pm 0.4}$}\\
		Hopper-M-R &-&-&60.9&\textbf{101.4$_{\pm 1.3}$}&94.7&\textbf{95.5$_{\pm 3.5}$}&83.0$_{\pm 1.7}$&\textbf{101.9$_{\pm 1.4}$}&101.6$_{\pm 0.7}$&\textbf{102.4$_{\pm 0.5}$}\\
		Walker2d-M-R &-&-&81.8&\textbf{87.1$_{\pm 2.6}$}&73.9&\textbf{83.0$_{\pm 2.9}$}&87.0$_{\pm 2.6}$&\textbf{94.7$_{\pm 2.3}$}&93.5$_{\pm 2.0}$&\textbf{95.3$_{\pm 2.3}$}\\
		\midrule
		Ave. (MuJoCo) &-&-&\cellcolor{yellow!30}75.3&\cellcolor{pink!30}\textbf{87.7} (\textcolor{blue}{+12.4})&\cellcolor{yellow!30}77.0&\cellcolor{pink!30}\textbf{82.0} (\textcolor{blue}{+5.0})&\cellcolor{yellow!30}82.4&\cellcolor{pink!30}\textbf{88.0} (\textcolor{blue}{+5.6})&\cellcolor{yellow!30}89.1&\cellcolor{pink!30}\textbf{90.9} (\textcolor{blue}{+1.8})\\
		\midrule
		AntMaze-M-P &58.8&\textbf{72.8$_{\pm 7.4}$}&0.0&\textbf{2.7$_{\pm 1.8}$}&71.2&\textbf{88.5$_{\pm 2.5}$}&73.3$_{\pm 6.2}$&72.1$_{\pm 10.9}$&-&-\\
		AntMaze-L-P &24.2&\textbf{28.3$_{\pm 7.3}$}&0.0&\textbf{1.2$_{\pm 0.6}$}&39.6&\textbf{53.3$_{\pm 5.4}$}&33.3$_{\pm 1.9}$&\textbf{51.4$_{\pm 9.2}$}&-&-\\
		AntMaze-M-D &46.8&\textbf{48.9$_{\pm 8.9}$}&0.8&\textbf{3.4$_{\pm 0.9}$}&70.0&\textbf{89.4$_{\pm 5.2}$}&52.7$_{\pm 1.9}$&\textbf{76.5$_{\pm 3.2}$}&-&-\\
		AntMaze-L-D &21.4&\textbf{31.2$_{\pm 12.3}$}&0.0&\textbf{0.8$_{\pm 0.4}$}&47.5&\textbf{51.9$_{\pm 8.8}$}&41.3$_{\pm 3.4}$&\textbf{50.4$_{\pm 8.2}$}&-&-\\
		\midrule
		Ave. (AntMaze) &\cellcolor{yellow!30}37.8&\cellcolor{pink!30}\textbf{45.3} (\textcolor{blue}{+7.5})&\cellcolor{yellow!30}0.2&\cellcolor{pink!30}\textbf{2.0} (\textcolor{blue}{+1.8})&\cellcolor{yellow!30}57.1&\cellcolor{pink!30}\textbf{70.8} (\textcolor{blue}{+13.7})&\cellcolor{yellow!30}50.2&\cellcolor{pink!30}\textbf{62.6} (\textcolor{blue}{+12.4})&-&-\\
		\midrule
		Kitchen-M &8.1&\textbf{25.3$_{\pm 5.1}$}&0.0&\textbf{19.4$_{\pm 3.9}$}&51.0&\textbf{57.5$_{\pm 4.5}$}&50.2$_{\pm 1.8}$&\textbf{64.0$_{\pm 4.6}$}&-&-\\
		Kitchen-P &18.9&\textbf{23.8$_{\pm 5.6}$}&0.0&\textbf{20.6$_{\pm 4.5}$}&46.3&\textbf{65.4$_{\pm 3.0}$}&40.8$_{\pm 1.5}$&\textbf{61.2$_{\pm 9.2}$}&-&-\\
		\midrule
		Ave. (Kitchen) &\cellcolor{yellow!30}13.5&\cellcolor{pink!30}\textbf{24.6} (\textcolor{blue}{+11.1})&\cellcolor{yellow!30}0.0&\cellcolor{pink!30}\textbf{20.0} (\textcolor{blue}{+20.0})&\cellcolor{yellow!30}48.7&\cellcolor{pink!30}\textbf{61.5} (\textcolor{blue}{+12.8})&\cellcolor{yellow!30}45.5&\cellcolor{pink!30}\textbf{62.6} (\textcolor{blue}{+17.1})&-&-\\
		\bottomrule
	\end{tabular}
	\label{Motivational_comparison}
\end{table*}

We evaluate PAR to answer: (1) Does PAR improve the behavior cloning offline reinforcement learning method? (2) Can the improved method achieve state-of-the-art results? (3) How do PAR's key components affect performance?

\subsection{Experimental Setting}
\textbf{Benchmarks: }We evaluated PAR on D4RL benchmark tasks \citep{fu2021d4rldatasetsdeepdatadriven}, a widely adopted standard for offline RL. Our evaluation spans various domains, including locomotion, navigation and manipulation tasks to demonstrate the method's generalizability. Specifically, we utilize: (1) \textbf{Gym-MuJoCo tasks} (v2) (HalfCheetah, Hopper, Walker2d), covering three dataset qualities: Medium, Medium-Replay, and Medium-Expert; (2) \textbf{AntMaze tasks} (v0), which require stitching suboptimal trajectories for sparse-reward navigation, evaluating on Medium and Large maps with Play and Diverse datasets; and (3) \textbf{FrankaKitchen tasks} (v0), a challenging high-dimensional manipulation domain, where we use the Mixed and Partial datasets.

\textbf{Baselines: }We combine PAR with some popular behavior cloning actor-critic methods to demonstrate the role of PAR. These methods are: (1) \textbf{BCQ} \citep{BCQ}, which restricts the action space to the support of the behavior policy through KL behavior cloning; (2) \textbf{TD3+BC} \citep{TD3_BC}, a minimalist approach that adds MSE behavior cloning regularization term to the policy update; (3) \textbf{IQL} \citep{IQL}, which performs implicit policy improvement via MLE behavior cloning. When combined with PAR's GSOAW, the advantage weights are multiplied by the GSOAW weights to obtain the policy loss weights; (4) \textbf{EDP} \citep{EDP}, a diffusion-based policy method that uses MSE behavior cloning; and (5) \textbf{SSAR} \citep{SSAR}, which adopts a selective state-adaptive regularization strategy and MSE behavior cloning to balance policy constraint and improvement.

To further demonstrate PAR's effectiveness, we compare the above behavior cloning actor-critic methods using PAR with other state-of-the-art offline RL methods: (1) \textbf{PG} \citep{PG}, which replaces the standard Gaussian prior of a behavior-cloned diffusion model with a learnable distribution optimized via behavior-regularized objectives to directly generate high-value trajectories; (2) \textbf{TAT} \citep{TAT}, which aggregates information from both historical and current trajectories in a dynamic tree-like structure to marginalize unreliable states and prioritize impactful nodes for decision-making; (3) \textbf{LoMAP} \citep{RefPlan}, a training-free method that projects guided samples onto a low-rank subspace approximated from offline datasets to prevent infeasible trajectory generation; (4) \textbf{RefPlan} \citep{RefPlan}, which employs local manifold approximation and projection for manifold-aware diffusion planning; (5) \textbf{FQL} \citep{FQL}, which leverages an expressive flow-matching policy to model complex action distributions by training a one-step policy with RL to avoid unstable recursive backpropagation; and (6) \textbf{QIPO} \citep{QIPO}, which introduces energy-weighted flow matching that directly learns energy-guided flows for offline RL.

\textbf{Implementation Details: } We implement PAR and all baselines using PyTorch. For all baseline algorithms (BCQ, TD3+BC, IQL, EDP, and SSAR), we adopt the standard hyperparameters reported in their original papers. Across all methods, PAR is activated after half of the total training epochs. For all experiments, we report results averaged over $5$ random seeds. All models are trained on NVIDIA Tesla A800 GPUs. On the MuJoCo dataset, considering that the distribution area of actions on medium-expert is not as wide as that on medium and medium-replay, we empirically set the temperature parameter of GSOAW on medium-expert, medium, and medium-replay to 0.3, 0.5, and 1, respectively. In both AntMaze and Kitchen data, the temperature parameter is uniformly set to 0.5. The implementation code is provided in the supplementary materials or is available at: \url{https://anonymous.4open.science/r/PAR-140E}.

\subsection{Experimental Results}
\begin{table*}[t]
	\centering
	\caption{Comparison of normalized scores on D4RL benchmark on MuJoCo. Bold values indicate the best performance in each row, and underlined values indicate the second best.}
	\setlength\tabcolsep{3.0pt}
	\footnotesize
	\renewcommand{\arraystretch}{1}
	\begin{tabular}{lccccccccccc}
		\toprule
		\textbf{Dataset} & \textbf{PG} & \textbf{TAT} & \makecell{\textbf{LoMAP}} & \textbf{RefPlan} & \makecell{\textbf{FQL}} & \textbf{QIPO} & \makecell{\textbf{TD3+BC}\\ \textbf{+PAR}} & \makecell{\textbf{IQL}\\ \textbf{+PAR}} & \makecell{\textbf{EDP}\\ \textbf{+PAR}} & \makecell{\textbf{SSAR}\\ \textbf{+PAR}} \\
		\midrule
		HalfCheetah-M-E & \cellcolor{yellow!30}\underline{95.2} & 92.5 & 91.1 & 92.8 & 90.1 & 94.5 & 95.1$_{\pm 0.8}$ & 93.4$_{\pm 1.0}$ & 93.1$_{\pm 0.6}$ & \cellcolor{pink!30}\textbf{96.2}$_{\pm 0.4}$ \\
		Hopper-M-E      & 110.4 & 109.4 & 110.6 & 57.8 & 86.2 & 108.0 & \cellcolor{yellow!30}\underline{111.6}$_{\pm 1.3}$ & 108.5$_{\pm 6.5}$ & 111.1$_{\pm 1.0}$ & \cellcolor{pink!30}\textbf{112.5}$_{\pm 0.3}$ \\
		Walker2d-M-E    & 109.4 & 108.8 & 109.2 & \cellcolor{pink!30}\textbf{114.0} & 100.5 & 110.9 & 111.7$_{\pm 0.4}$ & 111.0$_{\pm 0.5}$ &\cellcolor{yellow!30}\underline{112.1}$_{\pm 0.3}$&111.5$_{\pm 0.5}$\\
		\midrule
		HalfCheetah-M   & 45.6 & 44.3 & 45.4 & \cellcolor{pink!30}\textbf{74.6} & \cellcolor{yellow!30}\underline{60.1} & 54.2 & 49.5$_{\pm 0.2}$ & 48.3$_{\pm 0.1}$ & 55.1$_{\pm 0.4}$ &58.3$_{\pm 1.9}$\\
		Hopper-M        & 97.5 & 82.6 & 93.7 & 32.8 & 74.5 & 94.1 & \cellcolor{yellow!30}\underline{101.8}$_{\pm 2.6}$ & 71.6$_{\pm 4.1}$ & 86.1$_{\pm 7.4}$ &\cellcolor{pink!30}\textbf{102.7}$_{\pm 0.5}$\\
		Walker2d-M      & 82.3 & 81.0 & 79.9 & \cellcolor{yellow!30}\underline{91.6} & 72.7 & \cellcolor{pink!30}\textbf{97.3} & 86.0$_{\pm 0.5}$ & 82.3$_{\pm 4.6}$ & 89.7$_{\pm 0.4}$ &89.1$_{\pm 0.8}$\\
		\midrule
		HalfCheetah-M-R & 46.4 & 39.2 & 39.1 & \cellcolor{pink!30}\textbf{76.3} & \cellcolor{yellow!30}\underline{51.1} & 48.0 & 45.4$_{\pm 0.5}$ & 44.7$_{\pm 0.2}$ & 48.1$_{\pm 1.4}$ &49.9$_{\pm 0.4}$\\
		Hopper-M-R      & 91.3 & 95.3 & 97.6 & 82.6 & 85.4 & 101.3 & 101.4$_{\pm 1.3}$ & 95.5$_{\pm 3.5}$ & \cellcolor{yellow!30}\underline{101.9}$_{\pm 1.4}$ &\cellcolor{pink!30}\textbf{102.4}$_{\pm 0.5}$\\
		Walker2d-M-R    & 83.7 & 78.2 & 78.7 & 91.2 & 82.1 & 78.6 & 87.1$_{\pm 2.6}$ & 83.0$_{\pm 2.9}$ & \cellcolor{yellow!30}\underline{94.7}$_{\pm 2.3}$ &\cellcolor{pink!30}\textbf{95.3}$_{\pm 2.3}$\\
		\midrule
		Average (MuJoCo) & 84.6 & 81.3 & 82.8 & 79.3 & 78.1 & 87.4 & 87.7 & 82.0 & \cellcolor{yellow!30}\underline{88.0} & \cellcolor{pink!30}\textbf{90.9}\\
		\midrule
		\bottomrule
	\end{tabular}
	\label{SOTA_comparison}
\end{table*}

\textbf{Locomotion Domain (MuJoCo):} 
As shown in Table~\ref{Motivational_comparison}, PAR significantly enhances all baseline methods on MuJoCo tasks. On average, TD3+BC improves from 75.3 to 87.7 (+12.4 points), IQL from 77.0 to 82.0 (+5.0 points), EDP from 82.4 to 88.0 (+5.6 points), and SSAR from 89.1 to 90.9 (+1.8 points). Notably, Hopper-M-R shows the most dramatic gain for TD3+BC, surging from 60.9 to 101.4 (+40.5 points). On average across all baselines, PAR yields the largest improvements on medium-replay (+7.5), followed by medium (+6.8) and medium-expert (+4.4), consistent with the decreasing proportion of suboptimal actions in these datasets.

\textbf{AntMaze Domain:} 
As reported in Table~\ref{Motivational_comparison}, PAR consistently yields average improvements across all baselines on AntMaze tasks: BCQ improves from 37.8 to 45.3 (+7.5), IQL from 57.1 to 70.8 (+13.7), EDP from 50.2 to 62.6 (+12.4), and even the deterministic TD3+BC advances from 0.2 to 2.0 (+1.8). These gains demonstrate that PAR effectively breaks through the performance limits imposed by behavior cloning, enabling successful suboptimal trajectory stitching for long-horizon sparse-reward navigation.

\textbf{Kitchen Domain:} 
As shown in Table~\ref{Motivational_comparison}, PAR yields substantial average improvements across all baselines on high-dimensional FrankaKitchen manipulation tasks: BCQ improves from 13.5 to 24.6 (+11.1), IQL from 48.7 to 61.5 (+12.8), EDP from 45.5 to 62.6 (+17.1), and TD3+BC surges from 0.0 to 20.0 (+20.0). These robust gains confirm that PAR successfully guides policies toward better action sequences in complex robotic control, remarkably rescuing TD3+BC from complete failure.

\textbf{Comparison with Other SOTA Methods:} 
Table~\ref{SOTA_comparison} benchmarks PAR-enhanced methods against sophisticated trajectory planning and flow-matching SOTA methods. Specifically, SSAR+PAR achieves the highest average score of 90.9 on MuJoCo tasks, outperforming the strongest SOTA competitor QIPO (87.4) and securing top scores on five out of nine individual tasks. This confirms that plugging PAR into simple actor-critic methods yields competitive or superior performance compared to highly complex SOTA algorithms.

\subsection{Ablation Study}
To gain deeper insights into the contribution of each component, we conducted the following ablation study: (1) Necessity of introducing new actions; (2) Necessity of GSOAW. Besides, sensitivity of temperature $\tau$ and computational complexity analyses are provided in Appendix \ref{Analysis_of_T} and Appendix \ref{Complexity}.

\begin{table}[h]
	\begin{minipage}[t]{0.48\textwidth}
		\vspace{0pt}%
		\textbf{Necessity of New Actions}: To validate introducing new actions through the replacement, we compare PAR with its variants without new actions. Specifically, we consider: (1) \textbf{Base}, TD3+BC without any modification; (2) \textbf{Filtering}, when the direction of the action value function’s local ascent is inconsistent with the direction of the offline action’s behavior cloning (i.e., the cosine similarity is less than 0), the corresponding action is directly removed; (3) \textbf{PAR}, our full method that replaces low-value actions with GDAAR and GSOAW. Table~\ref{tab:necessity} shows that while Filtering improves over Baseline (78.7 vs. 75.3 average), PAR significantly outperforms both methods (87.7 average). This gap highlights that simply filtering suboptimal actions is insufficient, and introducing better actions is essential for breaking the performance ceiling. 
	\end{minipage}\hfill
	\begin{minipage}[t]{0.48\textwidth}
		\vspace{0pt}%
		\centering
		\caption{Ablation study on PAR components. Results are normalized scores averaged over 5 seeds. Bold represents the best result.}
		\footnotesize
		\setlength\tabcolsep{1.pt}
		\label{tab:necessity}
		\begin{tabular}{lcccc}
			\toprule
			\textbf{Dataset} & \textbf{Base} & \textbf{Filtering} & \makecell{\textbf{No}\\ \textbf{GSOAW}}& \textbf{PAR} \\
			\midrule
			HalfCheetah-M-E & 90.7 &90.9$_{\pm 1.0}$&74.1$_{\pm 3.9}$& \textbf{95.1$_{\pm 0.8}$} \\
			Hopper-M-E & 98.0 &102.4$_{\pm 5.4}$&78.5$_{\pm 5.7}$& \textbf{111.6$_{\pm 1.3}$} \\
			Walker2d-M-E & 110.1 &111.2$_{\pm 0.1}$&91.7$_{\pm 3.3}$&\textbf{111.7$_{\pm 0.4}$} \\
			\midrule
			HalfCheetah-M & 48.3 &48.2$_{\pm 0.2}$&49.0$_{\pm 0.3}$& \textbf{49.5}$_{\pm 0.2}$ \\
			Hopper-M & 59.3 &65.3$_{\pm 1.1}$&63.7$_{\pm 3.5}$& \textbf{101.8}$_{\pm 2.6}$ \\
			Walker2d-M & 83.7 &84.7$_{\pm 0.5}$&85.1$_{\pm 1.3}$& \textbf{86.0}$_{\pm 0.5}$ \\
			\midrule
			HalfCheetah-M-R & 44.6 &44.5$_{\pm 0.3}$&44.9$_{\pm 0.6}$&\textbf{45.4}$_{\pm 0.5}$ \\
			Hopper-M-R & 60.9 &76.5$_{\pm 13.7}$&95.2$_{\pm 8.0}$& \textbf{101.4}$_{\pm 1.3}$ \\
			Walker2d-M-R & 81.8 &85.2$_{\pm 1.5}$&86.5$_{\pm 3.2}$& \textbf{87.1}$_{\pm 2.6}$ \\
			\midrule
			Average & 75.3 & \cellcolor{yellow!30}78.7 (\textcolor{blue}{+3.4}) &74.3(\textcolor{blue}{-1.0})& \cellcolor{pink!30}\textbf{87.7} (\textcolor{blue}{+12.4})\\
			\bottomrule
		\end{tabular}
	\end{minipage}
	
	\vspace{0.2cm}

	\textbf{Necessity of GSOAW:} To evaluate the necessity of GSOAW, we compare PAR with a variant that solely utilizes GDAAR (denoted as ``No GSOAW''). As shown in Table~\ref{tab:necessity}, omitting GSOAW leads to severe performance degradation on \texttt{medium-expert} datasets, which are highly sensitive to out-of-distribution (OOD) actions, causing the overall average score to fall below the Baseline (74.3 vs. 75.3). Conversely, on \texttt{medium} and \texttt{medium-replay} datasets, relying exclusively on GDAAR still yields improvements over the Baseline, particularly achieving a significant boost on Hopper-M-R (from 60.9 to 95.2). These findings demonstrate that while GDAAR can independently discover higher-value actions, it suffers from severe OOD extrapolation errors on complex mixed-quality datasets. Thus, GSOAW is indispensable for mitigating OOD risks and ensuring robust performance.
\end{table}

\section{Discussion and Conclusion}
\textbf{Potential Impact, Limitations, and Future Work:} PAR provides a plug-and-play solution to break the performance ceiling of BC regularization in offline RL, shifting the paradigm from pure imitation to guided policy improvement. This approach may inspire new data-centric or self-training paradigms, helping offline RL algorithms surpass human-level decision-making by dynamically optimizing the supervision signal itself. However, PAR currently relies on the continuous nature of the action space to compute the local ascent direction via action-value gradients (GDAAR). Adapting this mechanism to discrete or hybrid action spaces would require further design, such as utilizing discrete advantage estimates.  Furthermore, PAR operates at the single-step transition level, which may not fully capture long-horizon dependencies. Therefore, future work could explore extending PAR to multi-step or trajectory-level replacements, potentially leveraging sequence models to generate coherent high-value action sequences, and broadening the framework's applicability to discrete action domains.

In this paper, we theoretically and empirically revealed that standard behavior cloning (BC) regularization in offline RL inherently imposes a performance ceiling by forcing indiscriminate imitation of suboptimal offline data. To break this ceiling, we proposed Proximal Action Replacement (PAR), an easy-to-use, plug-and-play training sample replacer. By substituting suboptimal dataset actions with better actions generated by a stable target policy---guided by the action-value function's local ascent direction (GDAAR) and adaptively bounded by value uncertainty (GSOAW)---PAR expands the effective action exploration space while maintaining training stability. Extensive experiments demonstrated that integrating PAR into existing BC-regularized offline RL algorithms consistently and significantly improves their performance across diverse continuous control and robotic manipulation tasks. We believe that PAR provides a new, data-centric perspective on overcoming the structural limitations of offline RL, paving the way for more optimal and robust algorithms.

\bibliography{neurips_2026}
\bibliographystyle{plainnat}


\newpage
\appendix
\section{Related Works}
\label{RelatedWork}
Existing offline RL methods can be broadly categorized as follows. \textbf{Behavior cloning regularization methods} constrain the learned policy to stay close to the behavior policy to prevent distributional shift. \textit{TD3+BC}: \cite{TD3_BC} adopts a minimalist approach that adds MSE behavior cloning regularization to policy updates. \textit{BRAC}: \cite{DBLP:journals/corr/abs-1911-11361} introduces a general framework that applies KL divergence regularization between the learned and behavior policies. \textit{SSAR}: \cite{SSAR} proposes selective state-adaptive regularization that adjusts regularization strength based on data quality, trusting Bellman updates in high-quality states while applying constraints selectively. \textit{IQL}: \cite{IQL} estimates the value of optimal actions implicitly through expectile regression and extracts the policy via advantage-weighted behavior cloning, avoiding explicit evaluation of out-of-distribution actions. \textit{ASPC}: \cite{jing2026adaptive} dynamically adjusts the policy constraint scale during training to adapt to various datasets without manual tuning. \textbf{Conservative value methods} address extrapolation errors by penalizing out-of-distribution actions in the value function. \textit{CQL}: \cite{NEURIPS2020_0d2b2061} learns a conservative Q-function that lower-bounds the true value function, while \cite{NEURIPS2021_f29a1797} extends this conservative principle to model-based offline RL. \textit{CPQL}: \cite{kim2026pengs} adapts the Peng's Q($\lambda$) operator for conservative value estimation to naturally induce implicit behavior regularization and mitigate over-pessimistic value estimation. \textbf{Generative model-based methods} leverage expressive policy classes to capture complex action distributions. \textit{BCQ}: \cite{BCQ} employs a VAE-based generative model to restrict the action space to the support of the behavior policy. \textit{Diffusion Policies}: \cite{wang2023diffusion} and \cite{EDP} represent policies as diffusion models to capture multimodal action distributions. \textit{Flow Matching Policies}: \cite{FQL} trains an expressive flow-matching policy by distilling iterative flow processes into one-step generation, while \cite{QIPO} introduces energy-weighted flow matching that directly learns energy-guided flows without auxiliary models. Additionally, \cite{nguyen2026onestep} enables one-step action generation without distillation by learning an average velocity field, and \cite{chae2026flow} leverages a flow behavior proxy policy to regularize both the actor and the critic to mitigate value overestimation. \textbf{Trajectory planning methods} generate complete trajectories for long-horizon decision-making. \textit{PG}: \cite{PG} proposes prior-guided diffusion planning that replaces the standard Gaussian prior with a learnable distribution. \textit{TAT}: \cite{TAT} aggregates information from historical and current trajectories in a dynamic tree structure to resist stochastic risks in diffusion planners. \textit{LoMAP}: \cite{RefPlan} projects guided samples onto a low-rank subspace to prevent infeasible trajectory generation in diffusion planners. \textit{TGCVG}: \cite{wang2026trajectory} proposes a trajectory-level data augmentation framework that synthesizes high-quality trajectories via conservative value-guided generation. \textbf{Conditional sequence modeling methods} learn action distributions based on history trajectories and target returns. \textit{DT}: \cite{DT} formulates reinforcement learning as a sequence modeling problem, where actions are generated conditioned on desired returns and past states. \textit{QCS}: \cite{QCS} proposes adaptive Q-aid for conditional supervised learning in offline RL, leveraging Q-values to guide sequence modeling. \textit{QT}: \cite{QT} combines the trajectory modeling ability of Transformers with the predictability of optimal future returns from dynamic programming methods by integrating Q-value maximization into the conditional sequence modeling training loss. \textit{QDFFDT}: \cite{wang2026offline} adaptively combines global sequence features and local immediate features through a learnable fusion mechanism to improve decision-making and trajectory synthesis.

\section{Proof of Theorem \ref{prop:suboptimality}}
\label{proof_of_3_1}
\begin{proof}
	We prove by contradiction. Since the optimization algorithm converges to a stationary point where the gradient vanishes, we examine whether this stationary point coincides with the optimal solution. Assume the stationary point is the optimal pair, i.e., $(\hat{\pi}, \hat{Q}) = (\pi^*, Q^*)$. For this to hold, the gradient of the actor objective must vanish at $\pi^*$: $\nabla_\theta J_{\theta}(\pi_{\theta}, Q^*)\big|_{\pi_{\theta}=\pi^*} = 0$.
	Expanding the gradient:
	\begin{equation}
		\nabla_{\theta} J_{\theta}(\pi_{\theta}, Q^*) \big|_{\pi_{\theta}=\pi^*} = \lambda \nabla_\theta Q^*(s, \pi_{\theta})\big|_{\pi_{\theta}=\pi^*} - \nabla_\theta \mathcal{L}_{\text{MSE}}(\pi_\theta, a_{\text{data}})\big|_{\pi_{\theta}=\pi^*}.
	\end{equation}
	Since $\pi^*$ is the unconstrained maximizer of $Q^*$, we have $\nabla_\theta Q^*(s, \pi^*) = 0$. Thus, the condition reduces to $\nabla_\theta \mathcal{L}_{\text{MSE}}(\pi^*, a_{\text{data}}) = 0$.
	However, under standard distance metrics, $\nabla_\theta \mathcal{L}_{\text{MSE}} = 0$ implies $\pi^*(s) = a_{\text{data}}$, which contradicts the premise that $a_{\text{data}}$ is sub-optimal.
	Therefore, the gradient at $\pi^*$ is non-zero, driving the policy update away from $\pi^*$. Consequently, the stationary policy $\hat{\pi}$ must differ from $\pi^*$. By the uniqueness of the maximum of $Q^*$, this implies $Q^*(s, \hat{\pi}(s)) < Q^*(s, \pi^*(s))$.
\end{proof}

\section{Extension of Theorem \ref{prop:suboptimality}}
\label{extension_of_3_1}

\subsection{Sub-optimality of KL Regularization}

\begin{theorem} \label{thm:KL_suboptimality}
\textnormal{\textbf{(Sub-optimality of KL BC Regularization)}} 
Let $(\hat{\pi}, \hat{Q})$ be the computed optimal pair of the alternating optimization for the KL-regularized objective $J_{\text{KL}}(\pi_{\theta}, Q) = \mathbb{E}_{s \sim \mathcal{D}}[\mathbb{E}_{a \sim \pi_{\theta}(\cdot|s)}[Q(s,a)] - D_{\text{KL}}(\pi_{\theta}(\cdot|s) \| \pi_{\beta}(\cdot|s))]$ and the Bellman error. Let $(\pi^*, Q^*)$ denote the ideal optimal pair where $\pi^*(a|s) = \delta(a - \arg\max_{a'} Q^*(s, a'))$. If the behavior policy $\pi_{\beta}$ has support on suboptimal actions (i.e., $\pi_{\beta}(a|s) > 0$ for some $a \neq \arg\max_{a'} Q^*(s, a')$), then the converged policy $\hat{\pi}$ is strictly sub-optimal:
\begin{equation}
\hat{\pi} \neq \pi^* \quad \text{and} \quad \mathbb{E}_{a \sim \hat{\pi}(\cdot|s)}[Q^*(s,a)] < \max_a Q^*(s,a).
\end{equation}
\end{theorem}

\begin{proof}
Consider the policy optimization objective with KL regularization subject to a behavior policy $\pi_{\beta}$:
\begin{equation}
	J_{\text{KL}}(\pi) = \mathbb{E}_{s \sim \mathcal{D}} \left[ \mathbb{E}_{a \sim \pi_{\theta}(\cdot|s)} [Q(s, a)] - D_{\text{KL}}(\pi_{\theta}(\cdot|s) \| \pi_{\beta}(\cdot|s)) \right].
\end{equation}
The closed-form optimal policy $\hat{\pi}$ that maximizes this objective is the Boltzmann distribution:
\begin{equation}
	\hat{\pi}(a|s) = \frac{\pi_{\beta}(a|s) \exp\big(Q(s, a)\big)}{Z(s)}, \quad \text{where } Z(s) = \int \pi_{\beta}(a'|s) \exp\big(Q(s, a')\big) da'.
\end{equation}
In contrast, the unregularized optimal policy $\pi^*$ is deterministic (or concentrates mass on optimal actions):
\begin{equation}
	\pi^*(a|s) = \delta\big(a - \arg\max_{a'} Q^*(s, a')\big).
\end{equation}

Since $\pi_{\beta}$ has support on suboptimal actions, there exists $a_{\text{sub}} \neq \arg\max_a Q^*(s,a)$ such that $\pi_{\beta}(a_{\text{sub}}|s) > 0$. Therefore, $\hat{\pi}(a_{\text{sub}}|s) = \frac{\pi_{\beta}(a_{\text{sub}}|s) \exp(Q(s, a_{\text{sub}}))}{Z(s)} > 0$.

Now, consider the expected Q-value under $\hat{\pi}$:
\begin{align}
\mathbb{E}_{a \sim \hat{\pi}(\cdot|s)}[Q^*(s,a)] &= \int \hat{\pi}(a|s) Q^*(s,a) da \\
&= \frac{1}{Z(s)} \int \pi_{\beta}(a|s) \exp(Q(s,a)) Q^*(s,a) da \\
&= \sum_{a} \hat{\pi}(a|s) Q^*(s,a).
\end{align}

Let $a^* = \arg\max_a Q^*(s,a)$ denote the optimal action. Since $\hat{\pi}$ assigns non-zero probability to suboptimal actions (specifically $a_{\text{sub}}$), and $Q^*(s, a_{\text{sub}}) < Q^*(s, a^*)$, we have:
\begin{align}
\mathbb{E}_{a \sim \hat{\pi}(\cdot|s)}[Q^*(s,a)] &= \hat{\pi}(a^*|s) Q^*(s,a^*) + \sum_{a \neq a^*} \hat{\pi}(a|s) Q^*(s,a) \\
&= \hat{\pi}(a^*|s) Q^*(s,a^*) + \hat{\pi}(a_{\text{sub}}|s) Q^*(s,a_{\text{sub}}) + \sum_{a \neq a^*, a \neq a_{\text{sub}}} \hat{\pi}(a|s) Q^*(s,a) \\
&< \hat{\pi}(a^*|s) Q^*(s,a^*) + \hat{\pi}(a_{\text{sub}}|s) Q^*(s,a^*) + \sum_{a \neq a^*, a \neq a_{\text{sub}}} \hat{\pi}(a|s) Q^*(s,a^*) \\
&= Q^*(s,a^*) \sum_{a} \hat{\pi}(a|s) \\
&= \max_a Q^*(s,a),
\end{align}
where the strict inequality holds because $\hat{\pi}(a_{\text{sub}}|s) > 0$ and $Q^*(s, a_{\text{sub}}) < Q^*(s, a^*)$.

Therefore, $\hat{\pi} \neq \pi^*$ (since $\hat{\pi}$ is a probability distribution while $\pi^*$ is a Dirac distribution) and $\mathbb{E}_{a \sim \hat{\pi}(\cdot|s)}[Q^*(s,a)] < \max_a Q^*(s,a)$, establishing strict sub-optimality.
\end{proof}

\subsection{Sub-optimality of MLE Regularization}

\begin{theorem} \label{thm:MLE_suboptimality}
\textnormal{\textbf{(Sub-optimality of MLE BC Regularization)}} 
Let $(\hat{\pi}, \hat{Q})$ be the computed optimal pair of the alternating optimization for the MLE-regularized objective $J_{\text{MLE}}(\pi_{\theta}, Q) = \mathbb{E}_{(s,a) \sim \mathcal{D}}[-w(s,a) \log \pi_{\theta}(a|s)]$ and the Bellman error, where $w(s,a) > 0$ is a bounded weighting function (i.e., $w(s,a) < \infty$ for all $(s,a)$). Let $(\pi^*, Q^*)$ denote the ideal optimal pair where $\pi^*(a|s) = \delta(a - \arg\max_{a'} Q^*(s, a'))$. If the dataset $\mathcal{D}$ contains suboptimal actions (i.e., there exists $(s,a) \in \mathcal{D}$ such that $a \neq \arg\max_{a'} Q^*(s, a')$), then the converged policy $\hat{\pi}$ is strictly sub-optimal:
\begin{equation}
\hat{\pi} \neq \pi^* \quad \text{and} \quad \mathbb{E}_{a \sim \hat{\pi}(\cdot|s)}[Q^*(s,a)] < \max_a Q^*(s,a).
\end{equation}
\end{theorem}

\begin{proof}
For MLE-based methods, the policy is updated by maximizing the weighted log-likelihood:
\begin{equation}
	\hat\theta = \arg\max_{\theta} \mathbb{E}_{(s, a) \sim \mathcal{D}} \left[ w(s, a) \log \pi_{\theta}(a|s) \right],
\end{equation}
where $w(s, a) > 0$ denotes the weight. The optimal solution $\hat{\pi}$ converges to the weighted behavior distribution:
\begin{equation}
	\hat{\pi}(a|s) = \frac{w(s,a) \pi_{\beta}(a|s)}{\sum_{a'} w(s,a') \pi_{\beta}(a'|s)}.
\end{equation}

Since $w(s,a) < \infty$ for all $(s,a)$ and the dataset contains suboptimal actions, the support of $\hat{\pi}$ is contained in the support of $\pi_{\beta}$: $\text{supp}(\hat{\pi}(\cdot|s)) \subseteq \text{supp}(\pi_{\beta}(\cdot|s))$.

Let $a^* = \arg\max_a Q^*(s,a)$ denote the optimal action. We consider two cases:

\textbf{Case 1:} If $\pi_{\beta}(a^*|s) = 0$ (optimal action not in dataset), then $\hat{\pi}(a^*|s) = 0$, and clearly $\hat{\pi} \neq \pi^*$ since $\pi^*(a^*|s) = 1$. Moreover, since $\hat{\pi}$ assigns zero probability to the optimal action, we have:
\begin{equation}
\mathbb{E}_{a \sim \hat{\pi}(\cdot|s)}[Q^*(s,a)] = \sum_{a \neq a^*} \hat{\pi}(a|s) Q^*(s,a) < Q^*(s,a^*) = \max_a Q^*(s,a),
\end{equation}
establishing strict sub-optimality.

\textbf{Case 2:} If $\pi_{\beta}(a^*|s) > 0$ but there exists $a_{\text{sub}} \neq a^*$ with $\pi_{\beta}(a_{\text{sub}}|s) > 0$ and $w(s,a_{\text{sub}}) > 0$ (which must hold since the dataset contains suboptimal actions), then:
\begin{equation}
\hat{\pi}(a_{\text{sub}}|s) = \frac{w(s,a_{\text{sub}}) \pi_{\beta}(a_{\text{sub}}|s)}{\sum_{a'} w(s,a') \pi_{\beta}(a'|s)} > 0.
\end{equation}

Since $\hat{\pi}$ assigns non-zero probability to suboptimal actions, we have:
\begin{align}
\mathbb{E}_{a \sim \hat{\pi}(\cdot|s)}[Q^*(s,a)] &= \sum_{a} \hat{\pi}(a|s) Q^*(s,a) \\
&= \hat{\pi}(a^*|s) Q^*(s,a^*) + \sum_{a \neq a^*} \hat{\pi}(a|s) Q^*(s,a) \\
&= \hat{\pi}(a^*|s) Q^*(s,a^*) + \hat{\pi}(a_{\text{sub}}|s) Q^*(s,a_{\text{sub}}) + \sum_{a \neq a^*, a \neq a_{\text{sub}}} \hat{\pi}(a|s) Q^*(s,a) \\
&< \hat{\pi}(a^*|s) Q^*(s,a^*) + \hat{\pi}(a_{\text{sub}}|s) Q^*(s,a^*) + \sum_{a \neq a^*, a \neq a_{\text{sub}}} \hat{\pi}(a|s) Q^*(s,a^*) \\
&= Q^*(s,a^*) \sum_{a} \hat{\pi}(a|s) \\
&= \max_a Q^*(s,a),
\end{align}
where the strict inequality holds because $\hat{\pi}(a_{\text{sub}}|s) > 0$ and $Q^*(s,a_{\text{sub}}) < Q^*(s,a^*)$.

Therefore, in both cases, $\hat{\pi} \neq \pi^*$ and $\mathbb{E}_{a \sim \hat{\pi}(\cdot|s)}[Q^*(s,a)] < \max_a Q^*(s,a)$, establishing strict sub-optimality.
\end{proof}

\section{Toy Experiment of Theorem \ref{prop:suboptimality}}
\label{proof_of_3_1_toy_experiment}
To empirically validate Theorem \ref{prop:suboptimality}, we designed a 2D continuous bandit task where the oracle Q-function is $Q(s, \mathbf{a}) = -\|\mathbf{a}\|^2_2$, with the optimal action $\mathbf{a}^*=[0, 0]^\top$. We constructed a suboptimal dataset where actions are sampled from a distribution $\mathcal{N}([2, 2]^\top, 1.0\mathbf{I})$, mimicking a biased behavior policy. The sample size of the offline dataset is 10000. We trained the policy using the classic TD3+BC algorithm as the backbone, where BC regularizations take three different forms: MSE, KL, and MLE. KL behavior cloning regularization simultaneously trains a behavior policy network $\pi_{\beta}(\cdot|s)$ \citep{DBLP:journals/corr/abs-1911-11361}. All experiments were trained for 10000 steps with a batch size of 256 and the Adam optimizer using a learning rate of 0.0003.

\begin{figure}[t]
		\includegraphics[width=\linewidth]{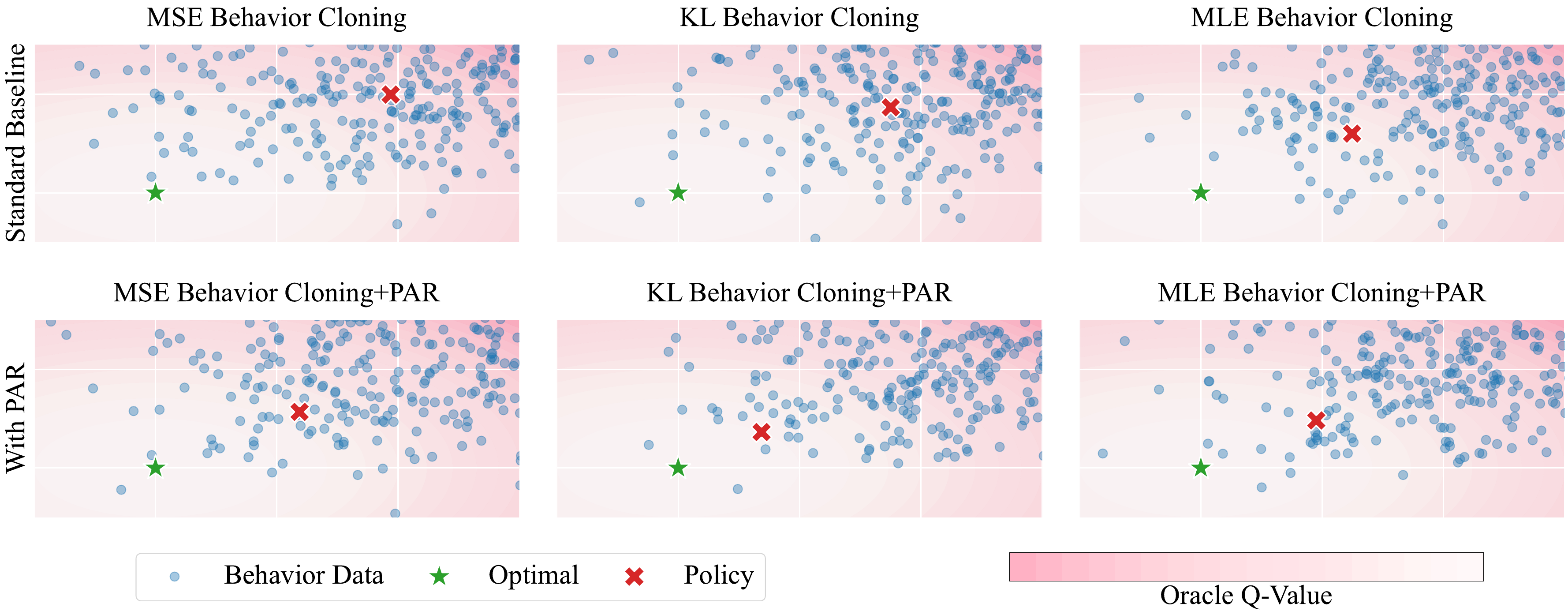}
		\caption{Offline RL experiments on a simple bandit task. The backbone algorithm is TD3+BC, where BC is set to three forms: MSE, KL, and MLE. The first row shows that behavior cloning loss prevents the learned policy from moving closer to the best policy. The second row shows that PAR can significantly alleviate this situation, making the learned policy closer to the optimal policy.}
		\label{fig:toy}
\end{figure}

Fig. \ref{fig:toy} shows the results of the toy experiment. Comparing the columns in the first row of Fig. \ref{fig:toy} reveals that the policy learned by MLE behavior cloning regularization is closer to the optimal policy than the policies learned by the other two behavior cloning regularizations, which may be due to the effect of the advantage function weighting. Nonetheless, the policy did not converge to the optimal action $\mathbf{a}^*$; instead, it converged to an intermediate location balanced between the optimal point and the center of the dataset cluster. This confirms that standard BC regularization structurally prevents the actor from fully exploiting the high-value regions identified by the critic, forcing a compromise between optimality and imitation.

\section{Proof of Theorem \ref{theo:GDAAR}}
\label{proof_of_GDAAR}
\begin{proof}
Fix a state $s$ and denote $a_{\theta}=\pi_{\theta}(s)$, $v=a-a_{\theta}$, $v'=\hat a-a_{\theta}$, and $g=\nabla_a Q(s,a)\vert_{a=a_{\theta}}$.
Under one BC-style actor step with step size $\eta>0$, write
$$
a_{\theta}^{+,\mathrm{BC}}=a_{\theta}+\eta v,\qquad
a_{\theta}^{+,\mathrm{GDAAR}}=a_{\theta}+\eta v'.
$$

Since $Q(s,\cdot)$ is $L$-smooth, for any displacement $\delta$,
$$
Q(s,a_{\theta}+\delta)\ge Q(s,a_{\theta})+g^\top\delta-\frac{L}{2}\|\delta\|_2^2,
$$
and
$$
Q(s,a_{\theta}+\delta)\le Q(s,a_{\theta})+g^\top\delta+\frac{L}{2}\|\delta\|_2^2.
$$
Applying the lower bound with $\delta=\eta v'$ and the upper bound with $\delta=\eta v$, we obtain
$$
\begin{aligned}
&Q(s,a_{\theta}^{+,\mathrm{GDAAR}})-Q(s,a_{\theta}^{+,\mathrm{BC}})\\
&\ge \eta\, g^\top(v'-v)-\frac{L\eta^2}{2}\left(\|v'\|_2^2+\|v\|_2^2\right).
\end{aligned}
$$

Now consider the first-order term. GDAAR selects $\hat a$ only when $\cos(g,v')>\cos(g,v)$. By definition,
$$
g^\top u=\|g\|_2\,\|u\|_2\,\cos(g,u),\qquad u\in\{v,v'\}.
$$
Motivated by the bounded nature of typical action spaces and the local constraint imposed by behavior cloning regularization, we assume local norm comparability in the form
$$
\|v'\|_2=(1+\xi)\|v\|_2,\qquad |\xi|=\mathcal{O}(\eta).
$$
Then
$$
g^\top(v'-v)
=\|g\|_2\,\|v\|_2\Big((1+\xi)\cos(g,v')-\cos(g,v)\Big)
\ge -C\eta
$$
for a local constant $C$, because $\cos(g,v')>\cos(g,v)$ and $|\xi|=\mathcal{O}(\eta)$.

Substituting into the previous inequality gives
$$
Q(s,a_{\theta}^{+,\mathrm{GDAAR}})-Q(s,a_{\theta}^{+,\mathrm{BC}})
\ge -\eta C\eta-\frac{L\eta^2}{2}\left(\|v'\|_2^2+\|v\|_2^2\right)
=-\mathcal{O}(\eta^2),
$$
hence
$$
Q(s,a_{\theta}^{+,\mathrm{GDAAR}})
\ge
Q(s,a_{\theta}^{+,\mathrm{BC}})-\mathcal{O}(\eta^2).
$$
This proves the claimed advantage up to second-order error.
\end{proof}

\section{Extension of Theorem \ref{theo:GDAAR}}
\label{extension_of_3_2}
\subsection{GDAAR for Vanilla KL Behavior Cloning}
\begin{theorem} \label{thm:KL_noninferiority}
\textnormal{\textbf{(Advantage Guarantee of GDAAR under KL Behavior Cloning)}} 
For a fixed state $s$, consider a Gaussian policy $\pi_{\theta}(\cdot|s)=\mathcal{N}(\mu_{\theta}(s),\sigma^2 I)$ with fixed $\sigma>0$ during one actor step. Let $\mu_{\theta}=\mu_{\theta}(s)$, $v=a-\mu_{\theta}$, $v'=\hat a-\mu_{\theta}$, and $g=\nabla_a Q(s,a)\vert_{a=\mu_{\theta}}$. Assume $Q(s,\cdot)$ is $L$-smooth, the actor uses a small KL-BC step, and the two candidate displacements are locally norm-comparable ($\|v'\|_2\approx\|v\|_2$). If GDAAR selects $\hat a$ (i.e., $\cos(g,v')>\cos(g,v)$), then the one-step value improvement under KL BC with GDAAR is no worse than vanilla KL BC up to second-order error:
\begin{equation}
Q(s,\mu_{\theta}^{+,\mathrm{GDAAR\text{-}KL}})
\ge
Q(s,\mu_{\theta}^{+,\mathrm{KL}})
-\mathcal{O}(\eta^2),
\end{equation}
where $\eta$ is the actor step size, and $\mu_{\theta}^{+,\mathrm{GDAAR\text{-}KL}}$ and $\mu_{\theta}^{+,\mathrm{KL}}$ denote one actor update step under GDAAR-KL and vanilla KL BC, respectively.
\end{theorem}

\begin{proof}
Fix a state $s$. For KL behavior cloning with Gaussian policies sharing covariance $\sigma^2 I$, minimizing $D_{\mathrm{KL}}(\mathcal{N}(\mu,\sigma^2 I)\,\|\,\mathcal{N}(m,\sigma^2 I))$ is equivalent to minimizing $\frac{1}{2\sigma^2}\|\mu-m\|_2^2$. Hence, one gradient step on the KL term induces a mean update of the form
\begin{equation}
\mu^+ = \mu + \eta_{\mathrm{eff}}(m-\mu),
\end{equation}
for some effective step size $\eta_{\mathrm{eff}}>0$ proportional to $\eta$.

Under vanilla KL BC, the supervision center is $m=a$, so
\begin{equation}
\mu_{\theta}^{+,\mathrm{KL}}=\mu_{\theta}+\eta_{\mathrm{eff}} v.
\end{equation}
Under GDAAR-KL, when the replacement rule chooses $\hat a$, the supervision center becomes $m=\hat a$, so
\begin{equation}
\mu_{\theta}^{+,\mathrm{GDAAR\text{-}KL}}=\mu_{\theta}+\eta_{\mathrm{eff}} v'.
\end{equation}

Now apply the same smoothness argument as in Theorem \ref{theo:GDAAR} at $a=\mu_{\theta}$:
\begin{equation}
Q(s,\mu_{\theta}^{+,\mathrm{GDAAR\text{-}KL}})-Q(s,\mu_{\theta}^{+,\mathrm{KL}})
\ge
\eta_{\mathrm{eff}}\, g^\top(v'-v)
-\frac{L\eta_{\mathrm{eff}}^2}{2}\left(\|v'\|_2^2+\|v\|_2^2\right).
\end{equation}
Because $\cos(g,v')>\cos(g,v)$ and $\|v'\|_2\approx\|v\|_2$ locally, the first-order term is non-inferior up to $\mathcal{O}(\eta_{\mathrm{eff}})$ perturbation, while the smoothness remainder is $\mathcal{O}(\eta_{\mathrm{eff}}^2)$. Therefore,
\begin{equation}
Q(s,\mu_{\theta}^{+,\mathrm{GDAAR\text{-}KL}})
\ge
Q(s,\mu_{\theta}^{+,\mathrm{KL}})
-\mathcal{O}(\eta_{\mathrm{eff}}^2)
=
Q(s,\mu_{\theta}^{+,\mathrm{KL}})
-\mathcal{O}(\eta^2).
\end{equation}
This proves the claimed advantage up to second-order error.
\end{proof}

\subsection{GDAAR for Vanilla MLE Behavior Cloning}
\begin{theorem} \label{thm:MLE_noninferiority}
\textnormal{\textbf{(Advantage Guarantee of GDAAR under MLE Behavior Cloning)}}
For a fixed state $s$, consider a Gaussian policy $\pi_{\theta}(\cdot|s)=\mathcal{N}(\mu_{\theta}(s),\sigma^2 I)$ with fixed $\sigma>0$ during one actor step. Let $\mu_{\theta}=\mu_{\theta}(s)$, $v=a-\mu_{\theta}$, $v'=\hat a-\mu_{\theta}$, and $g=\nabla_a Q(s,a)\vert_{a=\mu_{\theta}}$. Assume $Q(s,\cdot)$ is $L$-smooth, the actor uses a small weighted MLE-BC step with bounded positive weight $w(s,\cdot)$, and the two candidate displacements are locally norm-comparable ($\|v'\|_2\approx\|v\|_2$). If GDAAR selects $\hat a$ (i.e., $\cos(g,v')>\cos(g,v)$), then the one-step value improvement under MLE BC with GDAAR is no worse than vanilla MLE BC up to second-order error:
\begin{equation}
Q(s,\mu_{\theta}^{+,\mathrm{GDAAR\text{-}MLE}})
\ge
Q(s,\mu_{\theta}^{+,\mathrm{MLE}})
-\mathcal{O}(\eta^2),
\end{equation}
where $\eta$ is the actor step size, and $\mu_{\theta}^{+,\mathrm{GDAAR\text{-}MLE}}$ and $\mu_{\theta}^{+,\mathrm{MLE}}$ denote one actor update step under GDAAR-MLE and vanilla MLE BC, respectively.
\end{theorem}

\begin{proof}
Fix a state $s$. Under Gaussian policy $\pi_{\theta}(\cdot|s)=\mathcal{N}(\mu_{\theta}(s),\sigma^2 I)$, the weighted MLE loss for a supervision target $m$ is
\begin{equation}
\mathcal{L}_{\mathrm{MLE}}(\mu_{\theta};m)= -w(s,m)\log \pi_{\theta}(m|s).
\end{equation}
For fixed $\sigma$, this is equivalent (up to constants) to
\begin{equation}
\mathcal{L}_{\mathrm{MLE}}(\mu_{\theta};m)\equiv \frac{w(s,m)}{2\sigma^2}\|\mu_{\theta}-m\|_2^2.
\end{equation}
Hence, one gradient step induces a mean update
\begin{equation}
\mu^+ = \mu + \eta_{\mathrm{eff}}(m-\mu),
\end{equation}
where $\eta_{\mathrm{eff}}>0$ absorbs $\eta$, $\sigma^{-2}$, and bounded positive weight factors.

Under vanilla MLE BC, $m=a$, so
\begin{equation}
\mu_{\theta}^{+,\mathrm{MLE}}=\mu_{\theta}+\eta_{\mathrm{eff}} v.
\end{equation}
Under GDAAR-MLE, when the replacement rule chooses $\hat a$, $m=\hat a$, so
\begin{equation}
\mu_{\theta}^{+,\mathrm{GDAAR\text{-}MLE}}=\mu_{\theta}+\eta_{\mathrm{eff}} v'.
\end{equation}

Applying $L$-smoothness of $Q(s,\cdot)$ at $a=\mu_{\theta}$ yields
\begin{equation}
Q(s,\mu_{\theta}^{+,\mathrm{GDAAR\text{-}MLE}})-Q(s,\mu_{\theta}^{+,\mathrm{MLE}})
\ge
\eta_{\mathrm{eff}}\, g^\top(v'-v)
-\frac{L\eta_{\mathrm{eff}}^2}{2}\left(\|v'\|_2^2+\|v\|_2^2\right).
\end{equation}
Because $\cos(g,v')>\cos(g,v)$ and $\|v'\|_2\approx\|v\|_2$ locally, the first-order term is non-inferior up to $\mathcal{O}(\eta_{\mathrm{eff}})$ perturbation, while the second-order remainder is $\mathcal{O}(\eta_{\mathrm{eff}}^2)$. Therefore,
\begin{equation}
Q(s,\mu_{\theta}^{+,\mathrm{GDAAR\text{-}MLE}})
\ge
Q(s,\mu_{\theta}^{+,\mathrm{MLE}})
-\mathcal{O}(\eta_{\mathrm{eff}}^2)
=
Q(s,\mu_{\theta}^{+,\mathrm{MLE}})
-\mathcal{O}(\eta^2).
\end{equation}
This proves the claimed advantage up to second-order error.
\end{proof}

\section{Proof of Theorem \ref{theo:pao}}
\label{proof_of_proximal}
\begin{proof}
	The Mean Squared Error (MSE) loss for the critic is defined as $\mathcal{L}(Q) = \mathbb{E}_{(s,a)\sim\mathcal{D}}[(Q(s,a) - y)^2]$. By adding and subtracting the conditional expectation $\mathbb{E}[y|s,a]$, the loss can be decomposed as:
	\begin{equation}
		\mathcal{L}(Q) = \mathbb{E}_{(s,a)\sim\mathcal{D}} \left[ \underbrace{(Q(s,a) - \mathbb{E}[y|s,a])^2}_{\ge 0} + \text{Var}(y|s,a) \right].
	\end{equation}
	Since the first term is non-negative, the loss is minimized when $Q(s,a) = \mathbb{E}[y|s,a]$, and strictly lower-bounded by the conditional variance of the target $y$:
	\begin{equation}
		\min_Q \mathcal{L}(Q) = \mathbb{E}_{(s,a) \sim \mathcal{D}} [\text{Var}(y | s,a)].
	\end{equation}
	The target $y$ is computed using the current policy $\pi_{\theta}$: $y = r + \gamma Q(s', \pi_{\theta}(s'))$. Since the training data is generated by the behavior policy $\pi_{\beta}$, the ``on-distribution'' target should theoretically be $y_{\beta} = r + \gamma Q(s', \pi_{\beta}(s'))$. We can express the actual target $y$ as the on-distribution target plus an extrapolation error term $\Delta(s')$:
	\begin{equation}
		y = r + \gamma Q(s', \pi_{\beta}(s')) + \gamma \underbrace{\left( Q(s', \pi_{\theta}(s')) - Q(s', \pi_{\beta}(s')) \right)}_{\Delta(s', \pi_{\theta}, \pi_{\beta})} = y_{\beta} + \gamma \Delta(s').
	\end{equation}
	Applying the variance operator conditioned on $(s,a)$:
	\begin{equation}
		\text{Var}(y | s,a) = \text{Var}(y_{\beta} | s,a) + \gamma^2 \text{Var}(\Delta(s') | s,a) + 2\gamma \text{Cov}(y_{\beta}, \Delta(s') | s,a).
	\end{equation}
	The first term $\text{Var}(y_{\beta})$ represents the intrinsic aleatoric uncertainty of the environment, which is irreducible but bounded. The second term $\text{Var}(\Delta(s'))$ represents the epistemic uncertainty arising from querying the Q-network on actions $\pi_{\theta}(s')$ that deviate from the training distribution $\pi_{\beta}(s')$. Due to the following Cauchy-Schwarz inequality, the second term (quadratic in $\text{Var}(\Delta(s'))$) dominates the third term (the covariance):
	\begin{equation}
		|2\gamma \text{Cov}(y_{\beta}, \Delta(s') | s,a)| \le 2\gamma \sqrt{\text{Var}(y_{\beta} | s,a) \cdot \text{Var}(\Delta(s') | s,a)}.
	\end{equation}
	Specifically, since $\text{Var}(y_{\beta} | s,a)$ is bounded, when $\text{Var}(\Delta(s') | s,a)$ is sufficiently large (i.e., when policy divergence is large), the quadratic term $\gamma^2 \text{Var}(\Delta(s'))$ dominates the linear term $2\gamma \sqrt{\text{Var}(y_{\beta}) \cdot \text{Var}(\Delta)}$, establishing the relationship between critic loss and policy divergence.
	
	By the mean value theorem, since $Q(s', \cdot)$ is differentiable with respect to the action $a$, there exists $\bar{a}$ between $\pi_{\beta}(s')$ and $\pi_{\theta}(s')$ such that:
	\begin{equation}
		\Delta(s') = \nabla_a Q(s', \bar{a})^\top (\pi_{\theta}(s') - \pi_{\beta}(s')).
	\end{equation}
	During training, the critic learns a non-trivial value landscape with non-zero gradients in relevant directions. The variance of this term is lower-bounded by the squared magnitude of the action divergence via the Rayleigh quotient property:
	\begin{equation}
		\text{Var}(\Delta(s') | s,a) = \mathbb{E}_{s'|s,a} \left[ (\pi_{\theta} - \pi_{\beta})^\top \text{Cov}(\nabla_a Q(s', \bar{a})) (\pi_{\theta} - \pi_{\beta}) \right] \ge \mu \cdot \mathbb{E}_{s'|s,a} \left[ \|\pi_{\theta}(s') - \pi_{\beta}(s')\|_2^2 \right],
	\end{equation}
	where $\mu$ is the minimum eigenvalue of $\text{Cov}(\nabla_a Q(s', \bar{a}))$ over the relevant action space.
	Substituting this back into the lower bound:
	\begin{equation}
		\min_Q \mathcal{L}(Q) \ge \text{const} + \mu \cdot \mathbb{E}_{(s,a)\sim\mathcal{D}} \mathbb{E}_{s'|s,a} \left[ \|\pi_{\theta}(s') - \pi_{\beta}(s')\|_2^2 \right].
	\end{equation}
	Thus, when the actor $\pi_{\theta}$ significantly deviates from the behavior policy $\pi_{\beta}$ (i.e., $\|\pi_{\theta} - \pi_{\beta}\|$ is large), the lower bound of the critic loss explodes, causing training instability.
\end{proof}

\section{Proof of Theorem \ref{theo:GSOAW}}
\label{proof_of_GSOAW}
\begin{proof}
Fix iteration $t$. Define the critic loss under policy $\pi$ as
\begin{equation}
	\mathcal{L}_t(Q_{\phi_t};\pi)
	=
	\mathbb{E}_{(s,a)\sim\mathcal D}\!\left[
	\big(Q_{\phi_t}(s,a)-\big(r+\gamma Q_{\phi_t}(s',\pi(s'))\big)\big)^2
	\right].
\end{equation}
Let
\begin{equation}
z_t(s,a,s'):=Q_{\phi_t}(s,a)-\big(r+\gamma Q_{\phi_t}(s',\pi_\beta(s'))\big),
\end{equation}
\begin{equation}
\Delta_t(s'):=Q_{\phi_t}(s',\pi_{\theta_{t+1}}(s'))-Q_{\phi_t}(s',\pi_\beta(s')).
\end{equation}
Then
\begin{equation}
\mathcal{L}_t(Q_{\phi_t};\pi_{\theta_{t+1}})
=
\mathbb{E}\!\left[\big(z_t-\gamma\Delta_t\big)^2\right],
\quad
\mathcal{L}_t(Q_{\phi_t};\pi_{\beta})
=
\mathbb{E}\!\left[z_t^2\right],
\end{equation}
so the loss difference is
\begin{equation}
\label{eq:gsoaw_loss_diff_expand}
\begin{aligned}
\Delta\mathcal{L}_t
&:=\mathcal{L}_t(Q_{\phi_t};\pi_{\theta_{t+1}})-\mathcal{L}_t(Q_{\phi_t};\pi_{\beta}) \\
&=\mathbb{E}\!\left[\big(z_t-\gamma\Delta_t\big)^2-z_t^2\right] \\
&=\mathbb{E}\!\left[-2\gamma z_t\Delta_t+\gamma^2\Delta_t^2\right].
\end{aligned}
\end{equation}
Assume $|z_t|\le B_t$ almost surely at iteration $t$. Then
\begin{equation}
\label{eq:gsoaw_loss_diff_abs}
|\Delta\mathcal{L}_t|
\le
2\gamma B_t\,\mathbb{E}[|\Delta_t|]+\gamma^2\mathbb{E}[\Delta_t^2].
\end{equation}

Now use the GSOAW weight in Eq.~\eqref{eq:ood_weight}. For generated sample $i$, define
\begin{equation}
d_i:=|Q_{\phi_t}(s_i,\hat a_i)-V(s_i)|,\qquad
w_i:=\exp\!\left(-\frac{d_i}{\tau\,\sigma(s_i)+\varepsilon}\right)\in(0,1].
\end{equation}

Define $u_i\ge0$ as the displacement magnitude that would result from a
generated action if the weight were $1$ (i.e., from one unweighted BC
step on $\hat a_i$), and let
$\bar d = \max_{s\in\mathcal D}\,\|\pi_\theta(s)-\pi_\beta(s)\|_2$
be the maximal displacement that a single real-action BC step can incur.
Since the real-action BC term is unweighted and the critic weight
$w_i\le1$, the triangle inequality gives
$$
\|\pi_{\theta_{t+1}}(s_i')-\pi_\beta(s_i')\|_2
\le w_i u_i + \bar d,
$$
where $\xi_t:=\bar d$ absorbs the real-action drift. Assume $Q_{\phi_t}(s',a)$ is $L_Q$-Lipschitz in $a$. Hence
\begin{equation}
\label{eq:gsoaw_delta_bound}
|\Delta_t(s_i')|
\le
L_Q\|\pi_{\theta_{t+1}}(s_i')-\pi_\beta(s_i')\|_2
\le
L_Q\,(w_i u_i+\xi_t).
\end{equation}
Taking expectation in Eq.~\eqref{eq:gsoaw_delta_bound},
\begin{equation}
\mathbb{E}[|\Delta_t|]
\le
L_Q\,\mathbb{E}[w_i u_i+\xi_t],
\qquad
\mathbb{E}[\Delta_t^2]
\le
L_Q^2\,\mathbb{E}[(w_i u_i+\xi_t)^2].
\end{equation}
Substitute into Eq.~\eqref{eq:gsoaw_loss_diff_abs}:
\begin{multline}
\label{eq:gsoaw_final_bound}
\left|\mathcal{L}_t(Q_{\phi_t};\pi_{\theta_{t+1}})-\mathcal{L}_t(Q_{\phi_t};\pi_\beta)\right|
\le
2\gamma B_t L_Q\,\mathbb{E}[w_i u_i+\xi_t]
\;+\;
\gamma^2 L_Q^2\,\mathbb{E}[(w_i u_i+\xi_t)^2].
\end{multline}
Because $0\le w_i\le 1$, we have $w_i u_i\le u_i$ and $(w_i u_i)^2\le u_i^2$. Under $\mathbb{E}[u_i],\mathbb{E}[u_i^2]<\infty$, dominated convergence gives
\begin{equation}
\mathbb{E}[w_i u_i]\xrightarrow[\tau\to 0^+]{}0,\qquad
\mathbb{E}[(w_i u_i)^2]\xrightarrow[\tau\to 0^+]{}0.
\end{equation}
Hence, taking $\tau\to 0^+$ in Eq.~\eqref{eq:gsoaw_final_bound},
\begin{equation}
\limsup_{\tau\to 0^+}
\left|\mathcal{L}_t(Q_{\phi_t};\pi_{\theta_{t+1}})-\mathcal{L}_t(Q_{\phi_t};\pi_\beta)\right|
\le
2\gamma B_t L_Q\,\mathbb{E}[\xi_t]
\;+\;
\gamma^2 L_Q^2\,\mathbb{E}[\xi_t^2].
\end{equation}
Note that $\xi_t=\bar d$ is bounded and vanishes as the real-action
BC gradient step size tends to $0$; hence, by choosing a sufficiently
small BC step (or equivalently, a sufficiently large number of BC
repetitions with small weight), we can ensure
$2\gamma B_t L_Q\,\xi_t + \gamma^2 L_Q^2 \xi_t^2 \le \epsilon/2$. 
By the definition of limsup, there exists $\tau'$ > 0 such that for all $\tau < \tau'$,
the right-hand side of Eq.~\eqref{eq:gsoaw_final_bound}
is at most
\begin{equation}
\left|\mathcal{L}_t(Q_{\phi_t};\pi_{\theta_{t+1}})-\mathcal{L}_t(Q_{\phi_t};\pi_\beta)\right|
\le \epsilon.
\end{equation}
This completes the proof.
\end{proof}
\section{Extension of Theorem \ref{theo:GSOAW}}
\label{extension_of_3_4}
\subsection{Stability Guarantee of GSOAW under KL Behavior Cloning}
\begin{theorem}
\label{cor:gsoaw_kl}
\textnormal{\textbf{(Stability Guarantee of GSOAW under KL Behavior Cloning)}}
Consider the same setting as Theorem~\ref{theo:GSOAW}. Assume the policy is Gaussian, $\pi_\theta(\cdot|s)=\mathcal N(\mu_\theta(s),\sigma^2 I)$ with fixed $\sigma>0$ during one actor step, and the KL-BC term is
\begin{equation}
\mathcal{L}_{\mathrm{KL}}(\theta)=\mathbb E_{s\sim\mathcal D}\!\left[D_{\mathrm{KL}}\!\left(\pi_\theta(\cdot|s)\,\|\,\pi_\beta(\cdot|s)\right)\right].
\end{equation}
Then for any $\epsilon>0$, there exists $\tau>0$ such that
\begin{equation}
\left|\mathcal{L}_t(Q_{\phi_t};\pi_{\theta_{t+1}})-\mathcal{L}_t(Q_{\phi_t};\pi_\beta)\right|\le\epsilon.
\end{equation}
\end{theorem}

\begin{proof}
Under fixed covariance Gaussian policies, minimizing
$D_{\mathrm{KL}}(\mathcal N(\mu_\theta,\sigma^2 I)\,\|\,\mathcal N(m,\sigma^2 I))$
is equivalent to minimizing
$\frac{1}{2\sigma^2}\|\mu_\theta-m\|_2^2$.
Hence one small KL-BC actor step induces a mean update with the same local displacement form as weighted MSE BC, i.e., generated-action contribution is scaled by $w_i$ and real-action contribution stays unweighted around the behavior anchor.

Therefore, there exist $u_i\ge0$ and residual $\xi_t\ge0$ such that
\begin{equation}
\|\pi_{\theta_{t+1}}(s_i')-\pi_\beta(s_i')\|_2 \le w_i u_i+\xi_t,
\end{equation}
which is exactly the key bound used in Eq.~\eqref{eq:gsoaw_delta_bound}. The remainder follows identically to Theorem~\ref{theo:GSOAW}: by $L_Q$-Lipschitz continuity in action and Eq.~\eqref{eq:gsoaw_loss_diff_abs}--Eq.~\eqref{eq:gsoaw_final_bound},
\begin{equation}
\left|\mathcal{L}_t(Q_{\phi_t};\pi_{\theta_{t+1}})-\mathcal{L}_t(Q_{\phi_t};\pi_\beta)\right|
\le
2\gamma B_t L_Q\,\mathbb{E}[w_i u_i+\xi_t]
\;+\;
\gamma^2L_Q^2\,\mathbb{E}[(w_i u_i+\xi_t)^2].
\end{equation}
As $\tau\to0^+$, $w_i\to0$ on generated OOD samples and dominated convergence yields
$\mathbb E[w_i u_i]\to0$, $\mathbb E[(w_i u_i)^2]\to0$.
Choosing sufficiently small actor step (so $\xi_t$ is small) and then sufficiently small $\tau$ gives the desired $\epsilon$-bound.
\end{proof}

\subsection{Stability Guarantee of GSOAW under MLE Behavior Cloning}
\begin{theorem}
\label{cor:gsoaw_mle}
\textnormal{\textbf{(Stability Guarantee of GSOAW under MLE Behavior Cloning)}}
Consider the same setting as Theorem~\ref{theo:GSOAW}. Assume the policy is Gaussian, $\pi_\theta(\cdot|s)=\mathcal N(\mu_\theta(s),\sigma^2 I)$ with fixed $\sigma>0$ during one actor step, and the weighted MLE-BC term is
\begin{equation}
\mathcal{L}_{\mathrm{MLE}}(\theta)=\mathbb E_{(s,a)\sim\mathcal D}\!\left[-\omega(s,a)\log \pi_\theta(a|s)\right],
\end{equation}
where $\omega(s,a)>0$ is bounded. Then for any $\epsilon>0$, there exists $\tau>0$ such that
\begin{equation}
\left|\mathcal{L}_t(Q_{\phi_t};\pi_{\theta_{t+1}})-\mathcal{L}_t(Q_{\phi_t};\pi_\beta)\right|\le\epsilon.
\end{equation}
\end{theorem}

\begin{proof}
For Gaussian policy with fixed covariance, $-\log\pi_\theta(m|s)$ is equivalent (up to constants) to
$\frac{1}{2\sigma^2}\|\mu_\theta-m\|_2^2$.
Thus weighted MLE induces a local quadratic pull toward supervision target $m$, scaled by bounded positive factor $\omega(s,m)$. Under one small actor step, the generated-action component is multiplied by the GSOAW weight $w_i$, while the real-action component remains unweighted around the behavior anchor.

Therefore, as in Theorem~\ref{theo:GSOAW}, there exist $u_i\ge0$ and residual $\xi_t\ge0$ such that
\begin{equation}
\|\pi_{\theta_{t+1}}(s_i')-\pi_\beta(s_i')\|_2 \le w_i u_i+\xi_t.
\end{equation}
Applying the same Lipschitz-to-loss argument (Eq.~\eqref{eq:gsoaw_delta_bound} and Eq.~\eqref{eq:gsoaw_loss_diff_abs}--Eq.~\eqref{eq:gsoaw_final_bound}) gives
\begin{equation}
\left|\mathcal{L}_t(Q_{\phi_t};\pi_{\theta_{t+1}})-\mathcal{L}_t(Q_{\phi_t};\pi_\beta)\right|
\le
2\gamma B_t L_Q\,\mathbb{E}[w_i u_i+\xi_t]
\;+\;
\gamma^2L_Q^2\,\mathbb{E}[(w_i u_i+\xi_t)^2].
\end{equation}
Since $w_i=\exp\!\left(-d_i/(\tau\sigma(s_i)+\varepsilon)\right)$, as $\tau\to0^+$ we have
$\mathbb E[w_i u_i]\to0$ and $\mathbb E[(w_i u_i)^2]\to0$ by dominated convergence under finite first/second moments. Choosing sufficiently small actor step (small $\xi_t$) and then sufficiently small $\tau$ yields the desired $\epsilon$-bound.
\end{proof}

\section{Pseudocode}
\label{pseudocode}
Algorithm~\ref{alg:par_gdaar_gsoaw} provides the complete pseudocode for Proximal Action Replacement (PAR). At each training step, we first update the critic network and the value function, including its standard deviation head $\sigma_\psi$ (lines 5--6). Then, for each transition in the mini-batch, we compute the Q-value gradient $g$ with respect to the current policy's action $a_\theta$ (line 9). We evaluate the alignment of the real action $a$ and the generated target action $\hat{a}$ with the Q-gradient using cosine similarity (lines 10--11). If the generated action exhibits a higher cosine similarity with the Q-gradient than the real action, we replace the behavior cloning (BC) target with the generated action $\hat{a}$ and apply the Gaussian-Shaped OOD Action Weighting (GSOAW) $w$ to penalize potential out-of-distribution actions (line 12). Otherwise, we retain the real action $a$ as the BC target with a weight of $1$ (line 14). Finally, the policy is updated by combining the standard RL objective (e.g., maximizing the Q-value) with the dynamically weighted BC loss (line 18).

\begin{algorithm}[t]
	\caption{Proximal Action Replacement (PAR)}
	\label{alg:par_gdaar_gsoaw}
\begin{algorithmic}[1]
    \STATE \textbf{Initialize} actor $\pi_\theta$, critic $Q_\phi$, target actor $\pi_{\theta'}$, value function $V_\psi$ with standard deviation head $\sigma_\psi$, dataset $\mathcal{D}$, temperature $\tau$, stabilizers $\varepsilon$, target-network Polyak coefficient $\alpha$, and base offline RL hyperparameters.
    \STATE Define cosine similarity $\mathrm{Cos}(g,u):={g^{\top}u}/({\|g\|_{2}\|u\|_{2}+\varepsilon})$ as in Sec.~3.2.1.
    \FOR{$t=1, \dots, T$}
        \STATE Sample mini-batch $\mathcal{B} \sim \mathcal{D}$.
        \STATE Update $\phi$ with the standard critic loss on transitions in $\mathcal{B}$. \hfill \textit{\textcolor{gray}{Critic Training}}
        \STATE Update $\psi$ for $V_\psi$ (e.g., TD on $\mathcal{D}$) and train $\sigma_\psi$ with $\mathcal{L}_\sigma$ in Eq.~\eqref{eq:gnll_sigma} on $\mathcal{B}$.
        \FOR{each transition $(s,a,r,s')\in\mathcal{B}$}
            \STATE $\hat{a}\leftarrow\pi_{\theta'}(s)$,\quad $a_\theta\leftarrow\pi_\theta(s)$.
            \STATE $g\leftarrow \nabla_{a}Q_\phi(s,a)\big|_{a=a_\theta}$.
            \STATE $v\leftarrow a-a_\theta$,\quad $v'\leftarrow \hat{a}-a_\theta$.
            \IF{$\mathrm{Cos}(g,v')>\mathrm{Cos}(g,v)$}
                \STATE $a_{\mathrm{bc}}\leftarrow \hat{a}$;\quad $w\leftarrow \exp\!\big(-|Q_\phi(s,\hat{a})-V_\psi(s)|/(\tau\,\sigma_\psi(s)+\varepsilon)\big)$.
            \ELSE
                \STATE $a_{\mathrm{bc}}\leftarrow a$;\quad $w\leftarrow 1$.
            \ENDIF
            \STATE The weighted BC loss $\mathcal{L}_{\mathrm{bc}}$ is calculated using $a_{bc}$ as action labels and $w$ as weights.
        \ENDFOR
        \STATE Update $\theta$ with actor objective (e.g., maximize Q value) together with $\mathcal{L}_{\mathrm{bc}}$. \hfill \textit{\textcolor{gray}{Policy Training}}
        \STATE $\theta'\leftarrow \alpha\theta'+(1-\alpha)\theta$.
    \ENDFOR
\end{algorithmic}
\end{algorithm}

\section{Sensitivity Analysis of Temperature $\tau$}
\label{Analysis_of_T}
Table~\ref{temperature_sensitivity} reports the sensitivity of PAR to the temperature $\tau$ in GSOAW (Eq.~\eqref{eq:ood_weight}) across seven representative datasets, using TD3+BC as the backbone. The results demonstrate that PAR is robustly insensitive to $\tau$ across a wide range. On AntMaze and Kitchen, the normalized score varies by less than $0.5$ and $1.0$ points, respectively, over the entire $\tau\in[0.5,2.0]$ range. On MuJoCo, performance is also stable within the practical tuning range: Hopper-M and Hopper-M-R maintain near-optimal scores across $\tau\in[0.5,1.0]$ and $\tau\in[0.5,2.0]$, respectively; only Hopper-M-E, whose narrow expert action distribution amplifies OOD risk, shows clear sensitivity, degrading toward the No-GSOAW baseline (Table~\ref{tab:necessity}) when $\tau$ is excessively large. Overall, these results demonstrate that PAR is robustly insensitive to $\tau$ across a wide range of datasets and domains, indicating that $\tau$ can be easily configured in practice without extensive per-task tuning.

\begin{table*}[t]
	\centering
	\caption{Sensitivity of PAR to the temperature parameter $\tau$.}
	\setlength\tabcolsep{3.0pt}
	\footnotesize
	\renewcommand{\arraystretch}{1}
	\begin{tabular}{lccccccc}
		\toprule
		\textbf{Dataset} & $\tau$ = 0.1 & $\tau$ = 0.3 & $\tau$ = 0.5 & $\tau$ = 1.0 & $\tau$ = 1.5 & $\tau$ = 2.0 & Baseline (TD3+BC) \\
		\midrule
		Hopper-M-E &103.2$_{\pm 5.2}$&\textbf{111.6$_{\pm 1.3}$}&108.0$_{\pm 2.8}$&99.5$_{\pm 4.8}$&89.0$_{\pm 6.2}$&81.5$_{\pm 6.8}$&98.0\\
		Hopper-M &66.2$_{\pm 1.5}$&89.5$_{\pm 5.5}$&\textbf{101.8$_{\pm 2.6}$}&100.2$_{\pm 3.5}$&82.0$_{\pm 6.5}$&70.5$_{\pm 8.0}$&59.3\\
		Hopper-M-R &77.0$_{\pm 12.5}$&92.5$_{\pm 8.8}$&99.5$_{\pm 7.5}$&\textbf{101.4$_{\pm 1.3}$}&100.8$_{\pm 7.5}$&99.0$_{\pm 8.0}$&60.9\\
		AntMaze-M-P &1.4$_{\pm 0.9}$&2.1$_{\pm 1.3}$&\textbf{2.7$_{\pm 1.8}$}&2.6$_{\pm 1.7}$&2.5$_{\pm 1.6}$&2.3$_{\pm 1.7}$&0.0\\
		AntMaze-M-D &1.6$_{\pm 0.9}$&2.6$_{\pm 1.3}$&\textbf{3.4$_{\pm 0.9}$}&3.3$_{\pm 1.0}$&3.2$_{\pm 1.3}$&3.0$_{\pm 1.6}$&0.8\\
		Kitchen-M &8.8$_{\pm 5.5}$&15.5$_{\pm 4.8}$&\textbf{19.4$_{\pm 3.9}$}&19.2$_{\pm 4.0}$&19.0$_{\pm 4.8}$&18.5$_{\pm 5.2}$&0.0\\
		Kitchen-P &9.2$_{\pm 6.0}$&16.5$_{\pm 5.5}$&\textbf{20.6$_{\pm 4.5}$}&20.3$_{\pm 4.5}$&20.0$_{\pm 5.2}$&19.5$_{\pm 5.8}$&0.0\\
		\bottomrule
	\end{tabular}
	\label{temperature_sensitivity}
\end{table*}

\section{Computational Complexity Analysis}
\label{Complexity}
The computational complexity of the proposed Proximal Action Replacement (PAR) algorithm is highly comparable to standard offline actor-critic methods. PAR introduces three main additional operations: 1) It requires training a state-value network $V_\psi$ with a standard deviation head $\sigma_\psi$. While this introduces an additional forward and backward pass, this overhead is entirely acceptable and standard in the field, as many well-known and highly efficient algorithms (e.g., Implicit Q-Learning (IQL) \citep{IQL} and Soft Actor-Critic (SAC) \citep{SAC}) inherently maintain both action-value and state-value functions; 2) Computing the action gradient $g = \nabla_a Q_\phi(s, a)$ requires one backward pass through the critic network with respect to the input action. Note that this gradient is often already computed and can be reused during the deterministic policy gradient update; 3) Calculating the cosine similarities and the GSOAW weights involves only basic element-wise vector operations. These operations are highly parallelizable on modern GPUs and introduce virtually no noticeable computational bottleneck.

\textbf{Runtime Comparison:} To empirically validate the efficiency of PAR, we report the wall-clock training time of various baselines with and without PAR in Table~\ref{Runtime_comparison}. The experiments are conducted on a single NVIDIA Tesla A800 GPU with a fixed number of training steps. As shown, PAR introduces only a marginal computational overhead. For TD3+BC, which originally lacks a state-value network, PAR increases the runtime by merely $6\%$--$8\%$. For IQL, which inherently maintains a state-value function, the additional overhead drops to approximately $3\%$. For more computationally intensive diffusion-based methods like EDP and SSAR, the relative overhead of PAR is further diluted to around $2\%$. These results confirm that PAR is highly efficient and introduces virtually no noticeable bottleneck in practice.

\begin{table*}[t]
	\centering
	\caption{Runtime comparison with and without PAR. The unit is hours.}
	\setlength\tabcolsep{5.0pt}
	\renewcommand{\arraystretch}{1}
	\footnotesize
	\begin{tabular}{l|cc|cc|cc|cc}
		\toprule
		\textbf{Dataset} & \makecell{\textbf{TD3}\\\textbf{+BC}} & \makecell{\textbf{TD3+BC}\\ \textbf{+PAR}} & \textbf{IQL} & \makecell{\textbf{IQL}\\ \textbf{+PAR}} & \textbf{EDP} & \makecell{\textbf{EDP}\\ \textbf{+PAR}} & \textbf{SSAR} & \makecell{\textbf{SSAR}\\ \textbf{+PAR}}\\
		\midrule
		Hopper-M-E &2.43&2.58&5.14&5.29&20.38&20.87&25.11&25.61 \\
		Hopper-M &2.40&2.56&5.13&5.31&19.93&20.41&24.97&25.47 \\
		Hopper-M-R &2.71&2.92&4.77&4.93&17.81&18.26&23.54&24.03 \\
		\bottomrule
	\end{tabular}
	\label{Runtime_comparison}
\end{table*}


\newpage
\input{checklist.tex}

\end{document}

%% file: checklist.tex
\section*{NeurIPS Paper Checklist}

\begin{enumerate}

\item {\bf Claims}
    \item[] Question: Do the main claims made in the abstract and introduction accurately reflect the paper's contributions and scope?
    \item[] Answer: \answerYes{}
    \item[] Justification: The abstract and Section~\ref{Introduction} state the main contributions and scope---the BC regularization analysis, PAR (GDAAR and GSOAW), and offline RL experiments---in line with what the paper delivers.

\item {\bf Limitations}
    \item[] Question: Does the paper discuss the limitations of the work performed by the authors?
    \item[] Answer: \answerYes{}
    \item[] Justification: Discussion section explicitly states methodological limitations (dependence on continuous action spaces for GDAAR-style gradients, single-step transition scope versus long-horizon structure) and future-work directions. 

\item {\bf Theory assumptions and proofs}
    \item[] Question: For each theoretical result, does the paper provide the full set of assumptions and a complete (and correct) proof?
    \item[] Answer: \answerYes{}
    \item[] Justification: Section~\ref{Method} embeds assumptions with each theoretical result (Theorems~\ref{prop:suboptimality}--\ref{theo:GSOAW}); appendix sections reachable from cross-references provide complete proofs plus KL/MLE extensions and empirical toy validations.

    \item {\bf Experimental result reproducibility}
    \item[] Question: Does the paper fully disclose all the information needed to reproduce the main experimental results of the paper to the extent that it affects the main claims and/or conclusions of the paper (regardless of whether the code and data are provided or not)?
    \item[] Answer: \answerYes{}
    \item[] Justification: The Experimental Setting describes D4RL tasks, baseline choices, PAR hyperparameters (e.g., GSOAW temperature), optimizer/batch conventions, seeds, GPU type, pseudocode for PAR (Appendix), and anonymized implementation URL; benchmarks are standard offline RL suites from the cited D4RL paper.

\item {\bf Open access to data and code}
    \item[] Question: Does the paper provide open access to the data and code, with sufficient instructions to faithfully reproduce the main experimental results, as described in supplemental material?
    \item[] Answer: \answerYes{}
    \item[] Justification: D4RL follows the openly posted artifact described by Fu et al.\ (referenced in the bibliography); supplemental materials provide anonymized PAR code plus appendix pseudocode matching the hyperparameters listed in Section~\ref{Results}.

\item {\bf Experimental setting/details}
    \item[] Question: Does the paper specify all the training and test details (e.g., data splits, hyperparameters, how they were chosen, type of optimizer) necessary to understand the results?
    \item[] Answer: \answerYes{}
    \item[] Justification: The ``Experimental Setting'' portion of Section~\ref{Results} lists domains, datasets, baseline algorithms, PAR-specific temperatures, batch size, random-seed repeats, framework (PyTorch), adherence to baseline papers' defaults, optimizer notes for synthetic experiments (Appendix), and standard D4RL evaluation splits.

\item {\bf Experiment statistical significance}
    \item[] Question: Does the paper report error bars suitably and correctly defined or other appropriate information about the statistical significance of the experiments?
    \item[] Answer: \answerYes{}
    \item[] Justification: Major tables annotate entries with $\pm$ subscripts after averaging over five seeds (explicitly noted in Section~\ref{Results}), capturing initialization or sampling variability for the repeated training runs underpinning headline gains.

\item {\bf Experiments compute resources}
    \item[] Question: For each experiment, does the paper provide sufficient information on the computer resources (type of compute workers, memory, time of execution) needed to reproduce the experiments?
    \item[] Answer: \answerYes{}
    \item[] Justification: Section~\ref{Results} states experiments run on NVIDIA Tesla~A800 GPUs; Appendix~\ref{Complexity} reports wall-clock hours for representative curricula and summarizes PAR-induced overhead atop standard training budgets inherited from cited baselines.
    
\item {\bf Code of ethics}
    \item[] Question: Does the research conducted in the paper conform, in every respect, with the NeurIPS Code of Ethics \url{https://neurips.cc/public/EthicsGuidelines}?
    \item[] Answer: \answerYes{}
    \item[] Justification: Study uses simulations and standard benchmark datasets cited in bibliography without human-participant interventions; anonymity requirements are handled via templated anonymized repo links prescribed for submission.

\item {\bf Broader impacts}
    \item[] Question: Does the paper discuss both potential positive societal impacts and negative societal impacts of the work performed?
    \item[] Answer: \answerNA{}
    \item[] Justification: The contribution is foundational offline RL algorithm research evaluated on standard simulators and D4RL; it is not tied to a concrete deployment or application with direct, specific societal impact of the kind this item targets.
    
\item {\bf Safeguards}
    \item[] Question: Does the paper describe safeguards that have been put in place for responsible release of data or models that have a high risk for misuse (e.g., pre-trained language models, image generators, or scraped datasets)?
    \item[] Answer: \answerNA{}
    \item[] Justification: The contribution is an offline policy-learning module trained on curated simulators/existing D4RL assets; authors do not release large generative backbone models or fresh scraped datasets that fall into NeurIPS' high-risk release category.

\item {\bf Licenses for existing assets}
    \item[] Question: Are the creators or original owners of assets (e.g., code, data, models), used in the paper, properly credited and are the license and terms of use explicitly mentioned and properly respected?
    \item[] Answer: \answerYes{}
    \item[] Justification: D4RL, baseline algorithms, and software libraries (e.g., PyTorch) are credited to their original publications and repositories; we use them under the terms of those publicly released assets as established by the cited sources.

\item {\bf New assets}
    \item[] Question: Are new assets introduced in the paper well documented and is the documentation provided alongside the assets?
    \item[] Answer: \answerYes{}
    \item[] Justification: PAR is instantiated via appendix pseudocode, detailed training hyperparameters, runtime/complexity analysis, anonymized supplementary code link, and no human-derived personal data dependence.

\item {\bf Crowdsourcing and research with human subjects}
    \item[] Question: For crowdsourcing experiments and research with human subjects, does the paper include the full text of instructions given to participants and screenshots, if applicable, as well as details about compensation (if any)? 
    \item[] Answer: \answerNA{}
    \item[] Justification: All evaluations rely on scripted simulators/D4RL benchmarks without crowdsourced labels or recruited participants.

\item {\bf Institutional review board (IRB) approvals or equivalent for research with human subjects}
    \item[] Question: Does the paper describe potential risks incurred by study participants, whether such risks were disclosed to the subjects, and whether Institutional Review Board (IRB) approvals (or an equivalent approval/review based on the requirements of your country or institution) were obtained?
    \item[] Answer: \answerNA{}
    \item[] Justification: There are no participants or biometric/crowdsourced acquisitions; IRB disclosures are unnecessary for the simulator-only empirical protocol.

\item {\bf Declaration of LLM usage}
    \item[] Question: Does the paper describe the usage of LLMs if it is an important, original, or non-standard component of the core methods in this research? Note that if the LLM is used only for writing, editing, or formatting purposes and does \emph{not} impact the core methodology, scientific rigor, or originality of the research, declaration is not required.
    \item[] Answer: \answerNA{}
    \item[] Justification: LLMs are neither part of PAR nor invoked as experimental learners; checklist declaration is unnecessary per NeurIPS guidance unless authoring assistance must be volunteered elsewhere.

\end{enumerate}